\documentclass[10.5pt]{article}
\usepackage{mathtools}
\usepackage{booktabs}
\usepackage{tikz}
\usetikzlibrary{arrows.meta, positioning}

\usepackage{hyperref}
\usepackage[pagewise]{lineno}
\usepackage{amsmath,amssymb,amstext,amsfonts}
\usepackage{algorithm}
\usepackage{algorithmic}
\usepackage{relsize}
\usepackage{amsthm}
\usepackage{ulem}
\usepackage{cancel}
\usepackage{enumerate}
\usepackage[all,arc]{xy}
\usepackage{mathrsfs}
\usepackage{tcolorbox}
\usepackage{graphicx}
\usepackage{subcaption}
\usepackage{float}
\usepackage[margin=0.9 in]{geometry}
\usepackage{listings}
\usepackage{xcolor}
\usepackage{subcaption} 

\newtheorem{Problem}{Problem}

\usepackage{graphicx}
\usepackage[utf8]{inputenc}
\newtheorem{theorem}{Theorem}[section]
\newtheorem{lemma}[theorem]{Lemma}
\newtheorem{proposition}[theorem]{Proposition}

\theoremstyle{definition}
\newtheorem{example}{Example}[section]
\usepackage{systeme}
\usepackage{comment}
\usepackage{dsfont}

\tcbuselibrary{theorems}
\theoremstyle{definition}

\usepackage[framemethod=TikZ]{mdframed}

\newcounter{texercise}
\newwrite\solout
\def\openoutsol{\immediate\openout\solout\jobname.sol}  \def\writesol#1{\immediate\write\solout{\noexpand\processsol{\thetexercise}{#1}}}

\newcounter{mytheorem}[section] 
\tcbset{
theostyle/.style={fonttitle=\bfseries\upshape, fontupper=\slshape,
arc=0mm, colback=blue!5,colframe=blue!75!black}, defstyle/.style={fonttitle=\bfseries\upshape, fontupper=\slshape,
colback=red!10,colframe=red!75!black},
}

\newcommand{\hui}[1]{{\color{red} #1}}

\newcommand{\calL}{\mathcal{L}}
\newcommand{\bR}{\mathbb{R}}
\newcommand{\bRd}{\mathbb{R}^d}

\newcommand{\bP}{\mathbb{P}}

\newcommand{\E}{\mathbb{E}}

\newcommand{\bbP}{\mathbb{P}}

\newcommand{\mcal}[1]{\mathcal{#1}}

\newcommand{\sbb}{\text{Schr{\"o}dinger Bridge }}

\newcommand{\bsbracket}[1]{{\Big[ #1 \Big]}}
\newcommand{\bsparath}[1]{{\Big( #1 \Big)}}

\newcommand{\cbrace}[1]{{\lbrace{#1}\rbrace }}

\numberwithin{equation}{section}
\title{The Ensemble Schr{\"o}dinger Bridge filter for Nonlinear Data Assimilation}
\author{Hui Sun \footnote{The author is supported by the Research Development Fund (RDF) Xi’an Jiaotong-Liverpool University (XJTLU), No.: RDF-24-02-029} \thanks{
Department of Financial and Actuarial Mathematics, School of Mathematics and Physics, Xi’an Jiaotong-Liverpool University, Suzhou 215123, China. Email: Hui.Sun@xjtlu.edu.cn.} }
\date{}
\begin{document}
\maketitle
\begin{abstract}
This work introduces a novel nonlinear optimal filtering method, termed the Ensemble Schr{\"o}dinger Bridge nonlinear filter (EnSBF). The proposed filter combines the standard prediction step with a diffusion generative modeling based analysis step, thereby completing one full filtering update. It is derivative-free, training-free, and also parallelizable. Numerical experiments demonstrate that the proposed algorithm in its original form performs effectively for nonlinear dynamics and nonlinear observation processes, including chaotic systems in mildly high dimensions. The results also show that the method demonstrates competitive performance against the classical approaches such as the ensemble Kalman filter and particle filter across a range of tests with varying degrees of nonlinearity. At the cost of introducing bias, it is discovered that adopting the coordinate-wise localization technique enables EnSBF to achieve good state-estimation performance in high-dimensional systems with nonlinear observations. Future work will focus on addressing the sparse observation problem, making the algorithm more time efficient and developing a rigorous convergence theory.  
\end{abstract}
\textbf{keywords:} Diffusion generative models, optimal filtering, training-free/derivative-free generative models, Schr{\"o}dinger Bridge, stochastic optimal control, data assimilation.   
\section{Introduction}
Optimal filtering represents an important component of data assimilation task as it finds wide applications in weather forecasting, material sciences, biology, and finance \cite{feng2,feng3, Evensen, Ramaprasad}. Optimal filtering is the process of estimating the hidden true states based on observed noisy states. 

In science and engineering communities, the filtering procedure is commonly formulated within the two-step Bayesian prediction--analysis framework, whose
goal is to accurately characterize, or approximate, the prior and posterior filtering densities. While these target densities can sometimes be approximated using parametric methods, such as the classical Kalman filter, they are more often represented and propagated using \textit{ensemble methods}.

The workhorses of ensemble methods are the underlying simulation-based approaches. Two classical and widely used methods within this framework are the
\textit{ensemble Kalman filter} (EnKF) and the \textit{particle filter} (PF). Compared with the original Kalman filter, the EnKF employs an ensemble of
particles to estimate the statistical properties of the system and represents the filtering densities through Gaussian approximations. Consequently, the
target distribution is approximated within a Gaussian framework. This approach can be limiting, since the true filtering density may be non-Gaussian and even
multimodal. As a result, the EnKF is most effective for linear or mildly nonlinear dynamics, but may fail to accurately track highly nonlinear systems.

The PF approach, also known as the sequential Monte Carlo method, does not rely on Gaussian or other parametric distributional assumptions. Instead, it uses an ensemble of weighted particles to approximate the filtering densities. In this representation, all information about the filtering densities is encoded in the particle ensemble. Therefore, the quality of the particle approximation determines the accuracy of the density approximation; in particular, it is important whether the particles concentrate near the modes of the target density and whether they adequately capture its overall configuration. Although PFs often perform well in nonlinear settings, their performance typically deteriorates as the dimension increases due to the curse of dimensionality. In high-dimensional regimes, a very large number of particles is required to obtain a satisfactory approximation of the target distribution, which substantially limits the practical applicability of this approach.

Recently, by leveraging ideas from generative diffusion models \cite{ys1,ys4,ho1,sb2,sb9}, a new ensemble-based simulation approach, namely the \textit{ensemble score filter} (EnSF), has been developed and shown to be highly effective \cite{feng1,feng4,feng5,feng6}. Unlike particle filters, the EnSF encodes distributional information through a \textit{score function}. The updates of the approximated distributions in both the prediction and analysis
steps are therefore reduced to updates of the corresponding score functions. Particles are then generated through a diffusion process, or stochastic
differential equation (SDE), whose drift incorporates the learned or estimated score function. Although this approach has been shown to be effective and efficient in high-dimensional settings, it introduces a structural model error in the simulation procedure, as reported in \cite{feng6}. In particular, an
approximation is used in the construction of the posterior score, which leads to bias in tracking the filtering densities. This limitation becomes more
pronounced in lower-dimensional settings with noisy observations; see, for  example, Example~4 in Section~3.2.4 of \cite{feng6}.

In this work, still within the ensemble-based simulation framework, we design a new diffusion-assisted nonlinear filter, termed the \textit{Ensemble
Schr{\"o}dinger Bridge filter} (EnSBF). The proposed method performs the analysis step, namely particle generation, through a diffusion SDE that we refer
to as the \textit{Schr{\"o}dinger Bridge SDE} (SB SDE). This SDE is induced by the solution of the \textit{Schr{\"o}dinger Bridge problem} (SBP) \cite{leonard,chen1,sb8}. Based on the SBP formulation, various numerical and deep-learning-based algorithms have been developed \cite{sb1,sb2,sb3,sb5,sb6,sb8,sb9,pham1,sb10}. Of particular interest are \cite{pham1} and \cite{sb10}, where generative tasks can be performed without
explicitly training neural networks. In the present work, we further extend this training-free approach to the dynamic filtering setting and obtain a nonlinear  data filter.

In contrast to the EnSF, the proposed method is also derivative-free: throughout the entire update procedure, no automatic differentiation with respect to either the state variable or the observation variable is required.  Recently, the authors of \cite{feng1} proposed the iterative EnSF (IEnSF) \cite{feng7} to mitigate the structural error in the original EnSF score approximation. By representing the prior as a Gaussian mixture, they derive an analytical decomposition of the time-dependent posterior score and introduce an iterative procedure to approximate the remaining intractable terms. The method achieves improved posterior approximation under nonlinear observations while remaining within the diffusion-model framework. EnSBF in its original form introduces no additional model-form or surrogate approximation in the posterior update. Its accuracy may nevertheless deteriorate in high dimensions because of weight degeneracy. The remaining errors in the particle generation step arise from the temporal discretization of the artificial-time SDE, the Monte Carlo approximation of the drift, and the finite-ensemble approximation of the forecast/posterior distributions.

Extensive numerical experiments demonstrate that the EnSBF achieves competitive performance compared with the classical PF and EnKF in the lower dimensional regime. In particular, it provides a more reliable characterization of the filtering distribution than the EnSF in low-dimensional regimes. At the same time, we establish a direct theoretical connection between the SBP-induced SDE and the reverse SDE appearing in score-based diffusion models. Based on this connection, we argue that one could switch between the EnSBF and EnSF approaches depending on whether the filtering problem lies in a low- or high-dimensional setting. In the meantime, in the high dimensional regime, we deliberately introduce bias by using the coordinate localization technique so that EnSBF can perform effectively in high dimensions for state-estimation purpose. 

The goal of the present work is to introduce the EnSBF framework, connect it with other diffusion based filters, i.e. EnSF and provide numerical evidence across representative nonlinear filtering examples. Convergence analysis and more comprehensive comparison with other methods, such as the local ensemble transform Kalman filter (LETKF), the Iterative EnSF (IEnSF) and localized particle filters over other higher dimensional test problems will be left as future work.

The remainder of the paper is organized as follows. Section~2 provides a brief overview of the Schr{\"o}dinger bridge (SB) problem and its analytical solution, and introduces the numerical algorithm that forms the computational backbone of the EnSBF developed in this work. We also establish a direct connection between the reverse-time SDE used in score-based diffusion models and the controlled SDE associated with the SB problem. This connection clarifies the relationship between the proposed EnSBF and EnSF. Section~3 reviews the optimal filtering problem and presents the main algorithm proposed in this paper, followed by a series of numerical experiments involving systems with varying degrees of nonlinearity and state dimensions. Finally, we summarize the main findings and discuss directions for future research. The appendix collects the theoretical background on the Schr{"o}dinger bridge problem.

\section{The Schr{\"o}dinger bridge problem and the static algorithm}
 In this section, we introduce the Schr{\"o}dinger Bridge problem, present its solutions, and describe the algorithm derived from this formulation. Since the accuracy of a nonlinear filter depends critically on the reliability of the underlying simulation procedure, this procedure should neither rely on restrictive distributional assumptions, as in the EnKF, nor introduce structural model error, as in the approximation of the posterior score in EnSF. The static algorithm derived from the solution of the Schr{\"o}dinger Bridge problem leads to a simple one-step simulation procedure that introduces no such error. It therefore serves as the computational backbone for the dynamic nonlinear filter, namely the EnSBF which will be developed in the next section. We remark that this static algorithm is a simple transformation of the Schr{\"o}dinger--F{\"o}llmer sampler and has already been studied in \cite{pham1,sb10}. The detailed construction of the dynamic EnSBF is deferred to the next section.

We also present two simulation experiments based on this static algorithm. These examples demonstrate that the method can generate samples that are statistically close to the given data samples, even when the target data exhibit complex structures. We emphasize that the algorithm only uses Euler discretization, ensemble-based approximation of the relevant expectations and ensemble approximation of the posterior distributions and no additional structural approximation is introduced. Finally, we conclude this section by establishing the connection between the SB SDE and the reverse SDE arising in score-based diffusion models. This connection shows that the proposed EnSBF is closely related to the EnSF.

\subsection{The Schr{\"o}dinger Bridge Problem}
The goal of the Schr{\"o}dinger Bridge (SB) problem is to find a probability law on a path space so that it has the prescribed marginal distributions at the starting time and terminal time and that it minimizes the loss with respect to the reference measure in terms of the relative entropy. More specifically, let $\Omega=C([0,1]; \bR^d)$ be the space of continuous functions on $[0,1]$ valued in $\bR^d$. Let $X=\lbrace X_t \rbrace_{t\in[0,1]}$ be a canonical process on $\Omega$ where $X_t(\omega)=\omega_t, \omega \in \Omega$. The $\sigma$-field defined on $\Omega$ is the canonical $\sigma-$field $\mcal{F}=\lbrace \omega: (\omega_t)_{t \in [0,1]} \in B \ | \ B \in \mcal{B}(\bR^d) \rbrace$. Let  $\mcal{P}(\Omega)$ denote the space of probability measures on the path space $\Omega$. Denoting $\mathbb{W}_{\sigma^2}^x$ as the law of BM starting from $x$ with variance $\sigma^2$, the law of reversible Brownian motion is defined via $\mathbb{Q}_{\sigma^2}=\int \mathbb{W}_{\sigma^2}^x dx$ so that its marginal is the Lebesgue measure at each time $t$. Note that $\mathbb{Q}_{\sigma^2}$ is an unbounded measure on $\Omega$. The original version of the Schr{\"o}dinger's problem is formulated as follows: 
\begin{Problem}[\textbf{Schr{\"o}dinger Bridge Problem}]
Let $\Omega:=C([0,1];\mathbb{R}^d)$ be the path space and let
$\mathbb{Q}_{\sigma^2}$ denote the reversible Brownian reference measure. Given two probability measures $\nu,\mu\in\mathcal{P}(\mathbb{R}^d)$, find
$\mathbb{P}^*\in\mathcal{P}(\Omega)$ such that
\begin{align}
    \mathbb{P}^*
    \in
    \arg\min_{\substack{\mathbb{P}\in\mathcal{P}(\Omega)\\
    \mathbb{P}_0=\nu,\ \mathbb{P}_1=\mu}}
    \mathcal{H}\bigl(\mathbb{P}\mid \mathbb{Q}_{\sigma^2}\bigr),
\end{align}
and the time marginal $\mathbb{P}_t$ is defined by the pushforward
$\mathbb{P}_t(B):=(X_t\#\mathbb{P})(B)=\mathbb{P}(X_t^{-1}(B))$, $B\in\mathcal{B}(\mathbb{R}^d),\quad t\in[0,1]$. 
where the relative entropy is defined by
\begin{align}
\mathcal{H}(\bbP | \mathbb{Q}_{\sigma^2})=
\begin{cases}
    \int \ln \frac{d\bbP}{d \mathbb{Q}_{\sigma^2}} d\bbP , &\ \ \text{if } \bbP << \mathbb{Q}_{\sigma^2} \\ 
    \infty , & \ \ \text{Otherwise.}
\end{cases}
\end{align}
\end{Problem}

Theorem 3.1 and 3.2 in \cite{sb8} show that the SB problem also has a control formulation.
\begin{Problem}\label{problem_control}(\textbf{Control Formulation}) Find the control
    \begin{align}
        \alpha^*_t(x) \in \arg \min_{\alpha \in \mcal{U}} \frac{1}{2 \sigma^2} \E[\int^1_0 ||\alpha_t||^2 dt]
    \end{align}
under the constraint: $\calL({X_T})= \mu$ where 
\begin{align}\label{sde_og_sim1}
	dX_t = \alpha_t dt + \sigma dW_t , \ \ X_0 \sim \nu. 
\end{align}
$\mcal{U}$ is the collection of all adapted process such that 
\begin{align}
	\E\bsbracket{\int^T_0 |\alpha_s|^2 ds} < \infty. 
\end{align}
\end{Problem}
Problem~\ref{problem_control} is particularly useful because it provides a systematic simulation procedure once the optimal control $\alpha_t^*$ is available. More precisely, let $\mu:=p_{\mathrm{data}}$ denote the target distribution. Starting from an initial particle cloud approximating the initial density $q$, one evolves the particles according to the controlled SDE \eqref{sde_og_sim1}. The resulting particle ensemble at terminal time $T=1$ then provides an approximation of the target distribution $\mu$.

One particular useful case is when $\nu=\delta_0$ the Dirac delta measure at $0$, and $\mu$ is the target measure. When the diffusion $\sigma$ is set to $1$, the optimal control takes an explicit form (See Appendix 5.1)
 \begin{align}
        \alpha_t(x)= \nabla_x \log \E_{Z\sim \mu_W}[\frac{\mu}{\mu_W}(x+ \sqrt{T-t} Z)] \label{scheme1_alpha}
    \end{align} and $\mu_W$ denotes the density of the standard Gaussian measure.
Therefore, to sample from the target distribution $\mu$, it suffices to start from the initial distribution $\delta_0$ and simulate the controlled dynamics
    \begin{align}\label{scheme_sde}
        dX_t = \alpha_t(X_t)dt + \sigma dW_t , && X_0 \sim \delta_0(x), && T=1. 
    \end{align}
This is the famous Schr{\"o}dinger-F{\"o}llmer sampler. We will discuss another similar design based on \eqref{og_design_seq} in Example 2.1.

\subsubsection{Algorithm based on the solution to the SB problem.}
In this section, we discuss our algorithmic design based on \eqref{scheme1_alpha}, due to its simplicity and its suitability for Monte Carlo-based simulation. After a change of measure, equation \eqref{scheme1_alpha} can be equivalently written as: 
\begin{align}
    \alpha_t(x) &= \nabla_x \log \E_{Z\sim \mu_W}[\frac{\mu}{\mu_W}(x+ \sqrt{T-t} Z)] \nonumber\\ 
    &= \nabla_x \log \E_{Z\sim \mu}[\exp(-\frac{1}{2(T-t)}|Z-x|^2+\frac{1}{2 T} |Z|^2)] \\ 
    &=\frac{\E_{Z \sim \mu}[(Z-x) \exp(-\frac{1}{2(T-t)}|Z-x|^2+\frac{1}{2 T} |Z|^2)]}{(T-t)\E_{Z \sim \mu}[\exp(-\frac{1}{2(T-t)}|Z-x|^2+\frac{1}{2 T} |Z|^2)]}. \label{exact_drift}
\end{align}
The representation in \eqref{exact_drift} involves expectations with respect to the data measure $\mu$. This form is particularly useful in simulation when the target distribution is available only through samples, since the expectations can then be approximated directly by empirical averages over the particles.

On the other hand, the original formulation \eqref{scheme1_alpha} expresses the same drift through an expectation with respect to the standard Gaussian reference measure. This alternative representation is useful when the target measure $\mu$ is absolutely continuous with respect to the Lebesgue measure and its density is known explicitly. In that case, the likelihood ratio $\mu/\mu_W$ can be evaluated analytically, and the Gaussian expectation can be estimated by standard Monte Carlo sampling. We refer to \cite{jiao1} for a more detailed discussion.

To generate samples approximately distributed according to the target measure $\mu$, we discretize the diffusion process \eqref{scheme_sde} and approximate the expectations in \eqref{exact_drift} by Monte Carlo. More precisely, given an ensemble $\{Z^i\}_{i=1}^M$ of samples from $\mu$, we define
\begin{align}
    \alpha(t,x)
    \approx \frac{ \sum^{M}_{i=1} (Z^i-x) \exp(-\frac{1}{2(T-t)}|Z^i-x|^2+\frac{1}{2 T} |Z^i|^2)}{(T-t)\sum^{M}_{i=1}  \exp(-\frac{1}{2(T-t)}|Z^i-x|^2+\frac{1}{2 T} |Z^i|^2)}:=\alpha^M(t,x). \label{approx_drift_1} 
\end{align}
Here, $M$ denotes the size of the data ensemble.

The corresponding Euler--Maruyama approximation of \eqref{scheme_sde} is then
given by
\begin{align}
    X^{M,N}_{t_{n+1}}
    =
    X^{M,N}_{t_n}
    + \alpha^M\left(t_n,X^{M,N}_{t_n}\right)\Delta t
    + \Delta W_{t_n},
    \qquad
    X^{M,N}_{t_0}=0,
    \label{em_scheme}
\end{align}
where $\Delta W_{t_n}:=W_{t_{n+1}}-W_{t_n}\sim \mathcal N(0,\Delta t\,I_d)$.
 
\begin{algorithm}
\caption{Algorithm for producing generative samples based on \sbb solution.}\label{algorithm_drift}
\begin{algorithmic}[1]
\REQUIRE 
Initializing the following 
\begin{itemize}
    \item The data collection $\lbrace Z_i \rbrace_{i=1,...,M}$ which forms an empirical measure $\frac{1}{M} \sum^M_{i=1} \delta_{Z_i}$.
    \item The initial ensemble collection $\lbrace X^{M,N,\iota}_{t_0} \rbrace_{\iota=1,...,B}$ which all start from 0.
    \item Total number of temporal discretization $N$, with terminal time $T=1$. 
\end{itemize}
\FOR{$n=0,1,...,N-1$, $\iota \in \lbrace 1,2,...,B \rbrace$}
    \STATE{ 
For each particle, update its position based on the following dynamics: 
            \begin{align}\label{alg_1_sde}
	X^{M,N,\iota}_{t_{n+1}} =X^{M,N,\iota}_{t_{n}}+ \alpha^M(t_n,X^{M,N,\iota}_{t_{n}}) \Delta t + \Delta W_{t_n}, \ \ X^{M,N,\iota}_{t_0}=0.
      \end{align} 
    }
\ENDFOR
\RETURN
The collection of generated particles $\lbrace X^{M,N,\iota}_{1} \rbrace_{\iota=1,...,B}$.
\end{algorithmic}
\end{algorithm}
We note that a similar algorithm can be found in the literature \cite{pham1, sb10}. We provide two numerical below to demonstrate that such an algorithm performs well on dataset with complex distribution in high dimensions. 

The SDE \eqref{scheme_sde} can be viewed as inducing a sampling operator
\[
    A:\mathcal P(\bRd)\rightarrow \mathcal P(\bRd),
\]
which maps the prescribed target distribution $\mu$ to itself, in the sense that the terminal law of the solution to \eqref{scheme_sde} is exactly $\mu$. That is, $A\mu=\mu$.

In a similar manner, the numerical scheme \eqref{alg_1_sde} induces an empirical
operator $A^B:\mathcal P(\bR^d)\rightarrow \mathcal P(\bR^d),$  
defined by
\begin{align}
    A^B\mu =
    \frac{1}{B}\sum_{\iota=1}^{B} \delta_{X^{M,N,\iota}_1}, \label{empirical_operator}
\end{align}
where $X^{M,N,\iota}_1$ denotes the terminal state of the $\iota$-th independent simulation of \eqref{alg_1_sde}, and $B$ is the total number of simulated terminal particles.

In other words, the operator $A^B$ first uses samples from the target measure $\mu$ to construct the Monte Carlo approximation of the drift, then propagates particles through the time-discretized SDE, and finally returns the empirical measure of the generated terminal states. Thus, $A^B\mu$ serves as a particle approximation of the target distribution $\mu$.

To illustrate that the Schr{\"o}dinger Bridge data generator is able to produce new samples that are statistically close to the original dataset, we implement the discretized SDE \eqref{alg_1_sde} as described in Algorithm \ref{algorithm_drift}. The target distribution is then approximated by the empirical distribution of the terminal particles.

As a first example, we consider the two-dimensional Moons dataset generated by the \texttt{make\_moons} routine from the \texttt{sklearn} library. In this experiment, the original dataset consists of $1000$ data points, which are shown as orange dots in Figure \ref{fig:moons}. The artificial samples generated by the proposed Schr{\"o}dinger Bridge data generator are shown as blue dots. The generated ensemble has size $400$, and the discretized diffusion process is
simulated with $N=1024$ time steps. This preliminary example demonstrates that the proposed generator can successfully recover the main geometric structure of the target distribution.

As a second example, we consider the standard handwritten digits dataset. The full dataset contains $10000$ samples of handwritten digits, where each sample
is represented as an image of size $1\times 28\times 28$. This example is used to examine the performance of the proposed method in a higher-dimensional
setting. In our experiment, artificial digit samples are generated using only $100$ input samples. The results in Figure \ref{fig:handwritten} suggest that Algorithm \ref{algorithm_drift} can still produce visually reasonable samples even when only a small number of input data points is available.

\begin{figure}[htbp]
    \centering
    \begin{subfigure}[b]{0.35\textwidth}
    \centering
         \includegraphics[width=\textwidth]{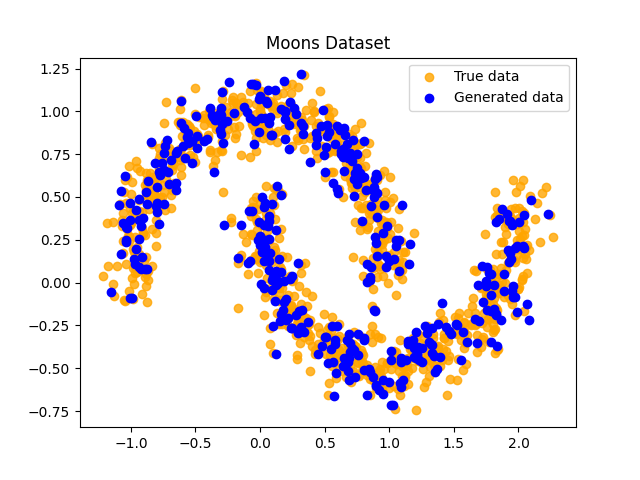} 
        \caption{Moons dataset.Total time $T=1.0$, and $N=1024$ total steps.} 
    \label{fig:moons}
    \end{subfigure}
    \vspace{1cm}
    \begin{subfigure}[b]{0.55\textwidth}
 \includegraphics[width=\textwidth]{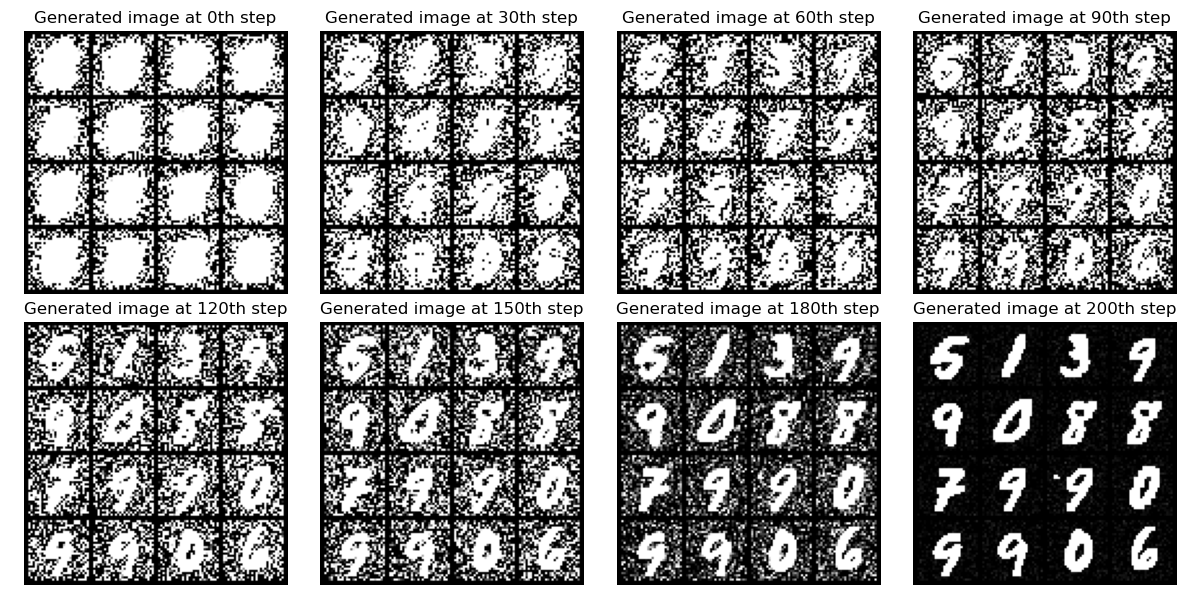} 
        \caption{Handwritten digit generation process using algorithm \ref{algorithm_drift}}
        \label{fig:handwritten}
    \end{subfigure}
    \caption{Demonstration for Algorithm \ref{algorithm_drift}. Two numerical examples are presented to show that the algorithm can capture the complicated structure of the underlying distribution (data) by producing samples which are statistically close to the original ones.}
    \label{fig:moon_hand}
\end{figure}

\subsection{Connecting the Schr{\"o}dinger Bridge problem and Score-based diffusion models}
Now we establish the connection between the solution of Schr{\"o}dinger bridge problem with score-based diffusion models. Throughout this subsection, we focus on the one-dimensional scalar setting for simplicity.  The extension to general scalar time-inhomogeneous coefficients follows from similar arguments.

Recall that for the score-based diffusion model \cite{tang1}, the forward and backward processes are as follows: 
\begin{align}
     d X_t &= \tilde{b}(t,X_t) dt + \tilde{\sigma}(t,X_t) dW_t, && X_0 \sim p_{data}, \ \ X_T \sim p(T,x) \\ 
     d Y_t & = (-\tilde{b}(T-t, Y_t) + \frac{\nabla \cdot (p(T-t,Y_t) \tilde{a}(T-t,Y_t))}{p(T-t,Y_t)}) dt + \tilde{\sigma}_{T-t} dW_t, && Y_0 \sim p(T,x), Y_T \sim p_{data}(x)
\end{align}
where $\tilde{a}:= \tilde{\sigma} \tilde{\sigma}^T$.

For analytic tractability, we consider the following linear SDE \eqref{reference_sde2_1} with the corresponding reverse SDE \eqref{reference_sde2_2} with $T=1$.  
\begin{align}
     d X_t &= \tilde{b}(t)X_t dt + \tilde{\sigma}(t) dW_t, && X_0 \sim p_{data}, \ \ X_1 \sim p(1,x) \label{reference_sde2_1} \\ 
     d Y_t &= \Big(- \tilde{b}(1-t)Y_t + \tilde{a}_{1-t} \nabla \log p(1-t,Y_t) \Big) dt + \tilde{\sigma}_{1-t} dW_t  , && Y_0 \sim p(1,x), \ \ Y_1 \sim p_{data}(x) \label{reference_sde2_2}
\end{align}
where both $\tilde{b}(t),\tilde{\sigma}(t)$ are time-dependent deterministic functions,  $p(1,x)= \int \tilde{q}(0,y,1,x)p_{data}(y)dy$. The solution of \eqref{reference_sde2_1} given $X_{t_0}=x$ is the following: 
\begin{align}
    X_t=e^{\int^t_{t_0} \tilde{b}_s ds} x + e^{\int^t_{t_0} \tilde{b}_s ds} \int^t_{t_0} e^{-\int^s_{t_0} \tilde{b}_r dr} \tilde{\sigma}_s dW_s
\end{align}
The transition density is given as 
\begin{align}
    \tilde{q}(t_0,x,t,y) = \frac{1}{\sqrt{2 \pi \tilde{\sigma}^2_{t_0,t}}} \exp \Big( - \frac{(y-\tilde{\mu}_{t_0, t}x)^2 }{2 \tilde{\sigma}^2_{t_0,t}} \Big)
\end{align}
where $\tilde{\mu}_{t_0, t}=e^{\int^t_{t_0} \tilde{b}_s ds}$ and $\tilde{\sigma}^2_{t_0,t}=e^{2 \int^t_{t_0} \tilde{b}_s ds} \int^t_{t_0} e^{-2\int^s_{t_0} \tilde{b}_r dr}\tilde{\sigma}^2_s ds = \int^t_{t_0} e^{2\int^t_{s} \tilde{b}_r dr}\tilde{\sigma}^2_s ds$.

\begin{theorem}
\label{thm:sb_score_connection}
Let $p(t,\cdot)$ denote the density of the solution to
\eqref{reference_sde2_1} with $X_0\sim p_{\mathrm{data}}$. Define
\[
    b_t:=-\tilde b_{1-t},
    \qquad
    \sigma_t:=\tilde\sigma_{1-t},
    \qquad
    a_t:=\sigma_t^2,
\]
and let $q(t,x,1,y)$ be the transition density of the reference SDE
\[
     dX_s=b_sX_s\,ds+\sigma_s\,dW_s,
     \qquad X_t=x.
\]
Then the optimal drift correction $\alpha^*$ in the stochastic control formulation of the Schr{\"o}dinger bridge problem satisfies
\begin{align}
    \alpha^*(t,x)
    =
    a_t\nabla_x\log p(1-t,x).
\end{align}
That is, the optimal Schr{\"o}dinger bridge control is the diffusion-scaled score of the forward diffusion model evaluated at time $1-t$.
\end{theorem}

\begin{proof}
Define: 
\begin{align*}
    \tilde{\mu}_{0,1-t} :=
    \exp\left(
        \int_0^{1-t}\tilde b_s\,ds
    \right), \qquad \nonumber \\ 
    \tilde{\sigma}^2_{0,1-t}
    :=
    \int_0^{1-t} \exp\left( 2\int_s^{1-t}\tilde b_r\,dr \right)
    \tilde{\sigma}_s^2\,ds .
\end{align*}
Consider the score function of $X_t$ which is the solution to \eqref{reference_sde2_1},
\begin{align}
    \nabla \log p(1-t,y)&=\nabla \log \int \tilde{q}(0,x,1-t,y) p_{data}(x) dx \nonumber\\ 
    &=\nabla \log \frac{1}{\sqrt{2 \pi \tilde{\sigma}^2_{0,1-t}}} \int \exp( -\frac{|y-\tilde{\mu}_{0,1-t} x|^2}{2 \tilde{\sigma}^2_{0,1-t}}) p_{data}(x)dx \nonumber\\ 
    &= \nabla \log \frac{1}{\sqrt{2 \pi \tilde{\sigma}^2_{0,1-t}}}\int \exp \Big( -\frac{|\tilde{\mu}^{-1}_{0,1-t}y- x|^2}{2 (\tilde{\sigma}_{0,1-t}/\tilde{\mu}_{0,1-t})^2 } \Big) p_{data}(x)dx \label{trans_2a}
\end{align}
Note that by a change of variable, we have:
\begin{align}
    \tilde{\mu}^{-1}_{0,1-t} & = \exp (-\int^{1-t}_{0} \tilde{b}_s ds) = \exp (\int^{1}_{t} -\tilde{b}_{1-r} dr) \nonumber \\ 
    \tilde{\sigma}^2_{0,1-t} & = \int^1_t e^{2 \int^t_s - \tilde{b}_{1-h}dh} \tilde{\sigma}^2_{1-s}ds\\
    \tilde{\sigma}^2_{0,1-t}/\tilde{\mu}^2_{0,1-t} &= \int^1_{t} e^{2\int^{1}_{s} -\tilde{b}_{1-r} dr} \tilde{\sigma}^2_{1-s} ds.
\end{align}
Now define the time-reversed reference coefficients by 
\begin{align}
    b_t & = -\tilde{b}_{1-t} , \ \ \ \sigma_t = \tilde{\sigma}_{1-t}.
\end{align}
We consider the reference SDE and denote its transition density kernel from time $t$ to $1$ by $q(t,x,1,y)$: 
\begin{align}\label{time_reversed_reference_sde}
    dX_s = b_s X_s ds + \sigma_s dW_s, && X_t=x. 
\end{align}
Then for the reference SDE \eqref{time_reversed_reference_sde} the transition
multiplier and transition variance from time $t$ to time $1$ are given by
\begin{align}
    \mu_{t,1}&=\exp (\int^{1}_{t} b_r dr) = \tilde{\mu}^{-1}_{0,1-t} \\ 
    \sigma^2_{t,1} & = \int^1_{t} e^{2\int^{1}_{s} b_r dr} \sigma^2_s ds =\tilde{\sigma}^2_{0,1-t}/\tilde{\mu}^2_{0,1-t}
\end{align}
As a result, we have the following sequence of equalities: 
\begin{align}
    \eqref{trans_2a} &= \nabla \log \int \frac{1}{\sqrt{2 \pi \sigma^2_{t,1} \mu^{-2}_{t,1}}} \exp \Big( -\frac{|\mu_{t,1}y- x|^2}{2 \sigma^2_{t,1} } \Big) p_{data}(x)dx \\
    & =  \nabla \log \int q(t,y,1,x) p_{data}(x) \mu_{t,1} dx \label{final_12}\\
    &= \nabla \log  \int q(t,y,1,x) p_{data}(x) dx  \label{final_2}
\end{align}
where the term $\mu_{t,1}$ is dropped inside the gradient operator since it only depends on time. 


Next, note that the SBP with $\mathbb{Q}^x$ induced by \eqref{time_reversed_reference_sde}, and the initial density $\rho_\nu(x)= \int q(0,x,1,y)p_{data}(y) e^{I_{0,1}}dy$ and the target density $p_{data}(y)$, \eqref{system_21} and \eqref{system_22} are satisfied with $f_0(x)=1$ and $g_1(y)=p_{data}(y) e^{I_{0,1}}$, where $I_{t_0, t}:= \int^t_{t_0} b_s ds$.
Indeed,
\begin{align}
    f_0(x)
    \int q(0,x,1,y)g_1(y)\,dy &=
    \int q(0,x,1,y)p_{\mathrm{data}}(y)e^{I_{0,1}}\,dy = \rho_\nu(x),
\end{align}
and, since $\int q(0,x,1,y)\,dx=e^{-I_{0,1}}$
we also have
\begin{align}
    g_1(y)
    \int q(0,x,1,y)f_0(x)\,dx &=
    p_{\mathrm{data}}(y)e^{I_{0,1}}e^{-I_{0,1}} = p_{\mathrm{data}}(y) = \rho_\mu(y).
\end{align}
Hence by Theorem 3.2 in \cite{sb8}, we have 
\begin{align}
    \alpha^*(t,x)&:= a_{t}\nabla \log \int q(t,x,1,y) p_{data}(y) e^{I_{0,1}} dy  \\ 
    &=a_{t} \nabla \log \int q(t,x,1,y) p_{data}(y) dy \label{final_3}\\
    &=a_{t} \nabla \log \int \tilde{q}(0,y,1-t,x) p_{data}(y) dy 
\end{align}
solves the corresponding optimal control problem and, in particular 
\begin{align}
    dX_t = \Big( b_t X_t + \alpha^*(t,X_t) \Big) dt + \sigma_t dW_t, && X_0 \sim \nu , \ \ X_1 \sim \mu 
\end{align}
with $\rho_{\nu}=\int q(0,x,1,y)p_{data}(y) e^{I_{0,1}}dy=\int \tilde{q}(0,y,1,x)p_{data}(y) dy , \rho_{\mu}(x)=p_{data}(x)$. 
\end{proof}

\begin{example}
We first consider an example in the original SBP setup, where we have $b_t=0$ and $\sigma_t = \sigma$ a positive constant. In applications, the empirical data distribution may be sparse or contain sharp spikes. A common strategy is therefore to first learn a smoothed, or regularized, version of the data distribution by adding Gaussian noise, and then recover the original data distribution through a denoising procedure. This provides part of the motivation
for \cite{ys4}. As such, we first take the target measure to be $p_{\sigma, data}:= \int p_{data}(y) \phi_{\sigma}(x-y)dy$ and will recover the target measure $p_{data}$ from the same. Hence we formulate the following two-step procedure: 
    \begin{align}
        dX_t &= \sigma^2 \nabla \log \E^{Z \sim \phi_\sigma}[\frac{p_{\sigma, data}}{\phi_\sigma}(X_t + \sqrt{1-t}Z)] dt + \sigma dW_t, && X_0=0 \label{scheme2_p1} \\ 
        dX_t &= \sigma^2 \nabla \log p_{\sqrt{1-t}\sigma, data}(X_t) dt + \sigma dW_t, && X_0 \sim p_{\sigma, data} \label{scheme2_p2}
    \end{align}
That is \eqref{scheme2_p1} first generates samples from the smoothed
distribution $p_{\sigma,\mathrm{data}}$ using a Schr{\"o}dinger--F{\"o}llmer-type SDE. Then in \eqref{scheme2_p2}, starting from $X_0 \sim p_{\sigma, data}$, samples are recovered by the reverse denoising dynamics. 

    Equation \eqref{scheme2_p1} can be obtained in a similar fashion as \eqref{part1_sys} by setting $\mu:=p_{\sigma, data}$ in the first step. Equation \eqref{scheme2_p2} is then obtained by setting $f_0(x)=1$ so that $g_1(y)=p_{data}(y)$. And one can check that $q(x)=\int p_{data}(y) h_{\sigma}(0,x,1,y) dy=p_{\sigma, data}$ is the smoothed data distribution. Note that \eqref{scheme2_p2} has recovered the reverse denoising process from the score-based diffusion model. 
\end{example}

\begin{example}
\textbf{(Recovering the set up for score-based diffusion models)}
    As a special case, in \cite{feng1}, the forward processes are taken to be 
    \begin{align}
        \tilde{b}_t & = \frac{d\log \alpha_t}{dt} \label{ensf_eq1} \\ 
        \tilde{\sigma}^2_t&=\frac{d \beta^2_t}{dt} - 2 \frac{d\log \alpha_t}{dt} \beta^2_t, && \alpha_t =1-t, \ \ \beta^2_t=t. \label{ensf_eq2}
    \end{align}
    where the benefit is that the initial data distribution will be transformed into pure noise (Gaussian distribution) at $T=1$. 
The forward SDE takes the form:
    \begin{align}
        dX_t = \tilde{b}_t X_t dt + \tilde{\sigma}_t dW_t, && \rho_{X_0} = p_{data} \label{ensf_sde}
    \end{align}
so the forward process satisfies
    \begin{align}
    X_t & =\alpha_t X_0+\beta_t Z \qquad Z\sim \mathcal N(0,1).
    \end{align}
    Therefore the initial data distribution is formally transported to
standard Gaussian noise at $T=1$. In this case,
\[
    \tilde b_t=-\frac{1}{1-t},
    \qquad
    \tilde\sigma_t^2=\frac{1+t}{1-t},
\]
so there is singularity near the terminal time $T=1$. This explains why in their implementation, the forward simulation is stopped at $1-\epsilon$ for some small $\epsilon >0$. 

Following the constructions above, for the Schr{\"o}dinger Bridge problem, the reference SDE is given by 
    \begin{align}
        d X_t = -\tilde{b}_{1-t} X_t dt + \tilde{\sigma}_{1-t}dW_t, && X_0=x.
    \end{align}
 And the controlled SDE with $\alpha^*(t,x)=\tilde{a}_{1-t}\nabla \log \int q(t,x,1,y) g_1(y) dy$ takes the form: 
    \begin{align}
        d X_t = \Big( -\tilde{b}_{1-t} X_t + \tilde{a}_{1-t}\nabla \log \int q(t,X_t,1,y) g_1(y) dy  \Big) dt + \tilde{\sigma}_{1-t}dW_t \nonumber.
    \end{align}
Taking $g_1(y)=p_{data}(y) e^{\int^1_0 -\tilde{b}_{1-t} dt}$, the exponential term in the SDE drops out due to the gradient operator. Hence: 
\begin{align}
    dX_t =
    \left(-\tilde b_{1-t}X_t+\tilde a_{1-t}\nabla_x \log \tilde p(1-t,X_t)
    \right)dt+\tilde\sigma_{1-t}\,dW_t.
\end{align}
Using $\Phi(x)$ to denote the standard Gaussian density, we also have $X_0 \sim \int q(0,x,1,y)p_{data}(y) e^{\int^1_0 -\tilde{b}_{1-t} dt }dy=\int \tilde{q}(0,y,1,x)p_{data}(y) dy=\Phi(x)$ and $X_T \sim p_{data}(y)$. Note that on the truncated interval $[0,1-\varepsilon]$, the corresponding factor  $e^{\int^{1-\epsilon}_0 -\tilde{b}_{1-t-\epsilon}dt}$ is finite. The corresponding reverse process starts from $p(1-\varepsilon,\cdot)$ and terminates at $p_{data}$. Formally taking $\varepsilon$ to 0 recovers the standard Gaussian terminal law of the forward process.
    Note that this is precisely the backward process for \eqref{ensf_sde}: 
    \begin{align}
         d X_t = \Big( -\tilde{b}_{1-t} X_t + \tilde{a}_{1-t} \nabla \log \tilde{p}(1-t,X_t)  \Big) dt + \tilde{\sigma}_{1-t}dW_t \label{main_obs}
    \end{align}
    with $X_0 \sim \Phi(x), X_T \sim p_{data}(x)$. 
\end{example}

As a consequence of the preceding discussion, under the linear scalar reference setting considered above, the reverse-time process of the score-based diffusion model can be identified with the controlled SDE associated with the solution of a Schr{\"o}dinger bridge problem under the stochastic optimal control formulation. At the same time, the Schr{\"o}dinger bridge
framework allows for a broader class of distribution generation procedures; see,
for example, \eqref{scheme1_alpha} and Example 2.1. Therefore, the score-based diffusion model may be viewed as a special case of the controlled Schr{\"o}dinger bridge dynamics. This observation provides a mathematical link between the EnSBF and EnSF methodologies and motivates the possibility of switching between the two approaches under different dimensional
regimes.

\section{Filtering problem, discrete setting}
In this section, we introduce a nonlinear filter which we call the Ensemble Schr{\"o}dinger Bridge filter based on the training-free Schr{\"o}dinger bridge data generator developed in the previous section.  Before presentation of the new algorithm, we briefly review the filtering problem and the classical particle method which serves as our main benchmark.

We consider a stochastic dynamical system of the following form:  
\begin{align} 
	 X_{j+1} &= f(X_j, \omega_j), && \textbf{Signal} \label{sb_pred1} \\
	Y_{j+1} &=g(X_{j+1}) + \epsilon_{j+1} \label{sb_analysis1}, && \textbf{Observation}
\end{align}
where $X \in \bR^d, Y \in \bR^n, \epsilon_j \in \bR^n,  \omega \in \bR^m$ and the noise terms $\omega, \epsilon$ are assumed to be independent Gaussian random variables. The goal of the filtering problem is to estimate $X_{j+1}$ given the observations $\mcal{Y}_{j+1}:=\cbrace{Y_l}_{l=0}^{j+1}$ which corresponds to observations of a single realized path. Mathematically, one endeavors to find the filtering density
\begin{align}
	\mathbb{P}(X_{j+1} |\mcal{Y}_{j+1}), && \forall j \in \lbrace 0, 1, ..., J-1 \rbrace 
\end{align}
where $J$ denotes the total number of filtering steps. 
A standard approach is the Bayesian filtering framework, in which one sequentially iterates between the following two steps. 
\begin{enumerate}[\textbullet]
    \item \textbf{Prediction step}.
    Given the posterior $\bbP(X_j|\mcal{Y}_j)$ at step $j$ and the transition kernel  $\bbP(X_{j+1}|X_j)$, the prior distribution at the $j+1$-th step is obtained based on the Chapman-Kolmogorov formula:
    \begin{align}\label{prediction_og}
        \bbP(X_{j+1} | \mcal{Y}_{j})=\int \bbP(X_{j+1}|X_j) \bbP(X_j|\mcal{Y}_{j}) d X_j.
    \end{align}
We denote this prior distribution by $\tilde{\mu}_{j+1}=\bbP(X_{j+1}|Y_{j})$. 
    \item \textbf{Update  step}.
    To obtain posterior density at step $j+1$, one combines the likelihood function  $\bbP(Y_{j+1}|X_{j+1})$ with the prior
    \begin{align}\label{update_og}
        \bbP(X_{j+1}|\mcal{Y}_{j+1}) \propto  \bbP(Y_{j+1}|X_{j+1})\bbP(X_{j+1}|\mcal{Y}_{j})
    \end{align}
This step is based on the Bayesian rule \cite{data_ass}, and 
\begin{align}
    \bbP(Y_{j+1}|X_{j+1}) \propto \exp \bsparath{-\frac{1}{2}(g(X_{j+1})-Y_{j+1})^T\Sigma^{-1} (g(X_{j+1})-Y_{j+1}))}
\end{align}
where $\Sigma = Cov(\epsilon)$ is the covariance matrix of the noise in \eqref{sb_analysis1}. 

We note that this step corresponds to a map: 
   $G_{j}: \mcal{P}(\bR^d) \rightarrow \mcal{P}(\bR^d)$ such that
\begin{align}\label{exact_sb_analysis}
	G_{j} \tilde{\mu}_{j+1}(dx) = \frac{\tilde{g}_j \tilde{\mu}_{j+1} (dx)}{\int_{\bR^d} \tilde{g}_j(x)\tilde{\mu}_{j+1} (dx)}:=\mu_{j+1}(dx)
\end{align}
where ${\mu}_{j+1}=\bbP(X_{j+1}|\mcal{Y}_{j+1})$ is the posterior distribution, and $\tilde{g}_j(X_{j+1}) \propto \bbP(Y_{j+1}|X_{j+1})$ is the likelihood associated with the new observation $Y_{j+1}$.
\end{enumerate}

\subsection{Design of the \sbb nonlinear filter}
To design the \sbb nonlinear filter, similar to the particle filter approach for the prediction step, we update the particle location based on dynamics \eqref{sb_pred1} to obtain the ensemble $ \lbrace \tilde{X}_{j+1} \rbrace$. And this ensemble provides an empirical approximation of the prior filtering distribution $\bP(X_{j+1}|\mcal{Y}_j) \approx \frac{1}{B}\sum^B_{i=1} \delta_{\tilde{X}^i_{j+1}}$.

For the analysis step, the target distribution is the posterior distribution $\mu_{j+1}$ defined in \eqref{exact_sb_analysis}. We generate samples approximately distributed according to this posterior using the \sbb data generator described in Algorithm \ref{algorithm_drift}.

To highlight the fact that the target distribution is $\mu_{j+1}$ at the $j$th filtering step, we write $\alpha(t,x,\mu_{j+1})$ to denote such dependence.  Using \eqref{exact_drift} the drift term can be written as follows: 
\begin{align}
    \alpha(t,x,\mu_{j+1})&= \frac{\int (z-x) \exp(-\frac{1}{2(T-t)}|z-x|^2+\frac{1}{2 T} |z|^2) d\mu_{j+1}(z)}{(T-t)\int \exp(-\frac{1}{2(T-t)}|z-x|^2+\frac{1}{2 T} |z|^2) d\mu_{j+1}(z)} \nonumber \\ 
    &= \frac{\int (z-x) \exp(-\frac{1}{2(T-t)}|z-x|^2+\frac{1}{2 T} |z|^2)\tilde{g}_j(z) d\tilde{\mu}_{j+1}(z)}{(T-t)\int \exp(-\frac{1}{2(T-t)}|z-x|^2+\frac{1}{2 T} |z|^2)\tilde{g}_j(z) d\tilde{\mu}_{j+1}(z)} \nonumber\\
    &= \frac{\E_{Z \sim \tilde{\mu}_{j+1}}[(Z-x) \exp(-\frac{1}{2(T-t)}|Z-x|^2+\frac{1}{2 T} |Z|^2)\tilde{g}_j(Z)]}{(T-t)\E_{Z \sim \tilde{\mu}_{j+1}}[\exp(-\frac{1}{2(T-t)}|Z-x|^2+\frac{1}{2 T} |Z|^2)\tilde{g}_j(Z)]} =: \tilde{\alpha}(t,x,\tilde{\mu}_{j+1}) \label{exact_drift_gtilde},
\end{align}
where in the second equality above, the normalizing constant $\int_{\bR^d} \tilde{g}_j(z)\tilde{\mu}_{j+1} (dz)$ was canceled. Using $\tilde{\alpha}$ as the drift, one obtains a map $\tilde{A}_j : \mcal{P}(\bR^d) \rightarrow \mcal{P}(\bR^d)$ such that $\tilde{A}_j \tilde{\mu}_{j+1}=\mu_{j+1}$.

Since $\tilde{\mu}_{j+1}$ is not known analytically, we approximate the
expectations in \eqref{exact_drift_gtilde} using samples from $\tilde{\mu}_{j+1}$.  More precisely, given predicted particles $\lbrace Z^i \rbrace^B_{i=1}$ $Z^i \sim \tilde{\mu}_{j+1}$, we use the empirical approximation: 
\begin{align}\label{app_r}
	\tilde{\alpha}(t,x,\tilde{\mu}_{j+1})\approx\frac{ \sum^{B}_{i=1} (Z^i-x) \tilde{g}_j(Z^i) \exp(-\frac{1}{2(T-t)}|Z^i-x|^2+\frac{1}{2 T} |Z^i|^2)}{(T-t)\sum^{B}_{i=1} \tilde{g}_j(Z^i) \exp(-\frac{1}{2(T-t)}|Z^i-x|^2+\frac{1}{2 T} |Z^i|^2)}.
\end{align}

Based on the above discussion, we can approximate target distribution by its empirical distribution with $M$ particles. This defines an empirical approximation $\tilde{A}^{N,M}_j: \mcal{P} (\bR^d) \rightarrow \mcal{P}(\bR^d)$ for the map \eqref{exact_sb_analysis} with $\tilde{A}^{N,M}_j \tilde{\mu}_{j+1}=\frac{1}{M}\sum^M_{\iota=1} \delta_{V^{N,\iota}_{T}}:=\frac{1}{M}\sum^M_{\iota=1} \delta_{X^\iota_{j+1}}$. Here $V^{N,\iota}_{T}$ is the Euler approximation of the solution of the SDE \eqref{scheme_sde}: 
\begin{align}
	V^{N,\iota}_{\tau_{l+1}} =V^{N,\iota}_{\tau_{l}}+ \tilde{\alpha}(\tau_l,V^{N,\iota}_{\tau_{l}},\tilde{\mu}^B_{j+1}) \Delta \tau + \Delta W_{\tau_l}, \ \ V^{N,\iota}_{\tau_0}=0, \ \iota \in \lbrace 1,2,...,M \rbrace. \label{discrete_drift_1}
\end{align}
with $T=1$, and $\Delta \tau := T/N$. Note that in \eqref{discrete_drift_1}, $V^N_{\tau_{l}}$ is purely an auxiliary process used only for the analysis step. $\tilde{\mu}^B_{j+1}$ is used for approximation of $\tilde{\mu}_{j+1}$ where the former stands for the ensemble particle approximation of the latter. 

Based on the preceding discussion, we summarize the proposed nonlinear filtering procedure in Algorithm \ref{algorithm_EnSBF}. We refer to this method as the Ensemble Schr{\"o}dinger Bridge Filter. 
\begin{algorithm}
\caption{Algorithm for the ensemble \sbb filter (EnSBF) }\label{algorithm_EnSBF}
\begin{algorithmic}[1]
\REQUIRE 
Initializing the following terms
\begin{itemize}
    \item The model: $f, g$ as in \eqref{sb_pred1} and \eqref{sb_analysis1}. The initial density $\mu_0 := \bbP(X_{0}|Y_{0}):=\frac{1}{B} \sum^B_{i=1} \delta_{X^i_{0}}$ which is a particle ensemble.
    \item Total number of filtering steps $J$. The time horizon $T=1$, total number of temporal discretization $N$ in the Euler scheme for the \sbb data generator. $\Delta \tau := \frac{T}{N}$.
\end{itemize}
\FOR{$j=0,1,2,..., J-1$}
    \STATE{ 
    \begin{itemize}
       \item Obtain $\bbP(X_{{j+1}}|\mcal{Y}_{j})=\frac{1}{B} \sum^B_{i=1} \delta_{\tilde{X}^i_{{j+1}}} $ based on $\bbP(X_{j}|\mcal{Y}_{j})=\frac{1}{B} \sum^B_{i=1} \delta_{X^i_{j}}$ using equation \eqref{sb_pred1}. 
       \item Update $\bbP(X_{{j+1}}|{\mcal{Y}_j})$ to $\bbP(X_{{j+1}}|\mcal{Y}_{{j+1}})$ by generating the particle ensembles $\frac{1}{B} \sum^B_{i=1} \delta_{X^i_{{j+1}}}$ via the Schr{\"o}dinger Bridge map $\tilde{A}^{N,B}_j: \mcal{P}(\bR^d) \rightarrow \mcal{P}(\bR^d)$: 
       \begin{align}
       	\tilde{A}^{N,B}_j \bbP(X_{{j+1}}|\mcal{Y}_{j}) \approx \bbP(X_{{j+1}}|\mcal{Y}_{{j+1}})
       \end{align} 
       which is realized by collecting the terminal state $V^{N,\iota}_T$ of the solution of the following discrete SDE: for $l=0,...,N-1$
       \begin{align}
	V^{N,\iota}_{\tau_{l+1}} =V^{N,\iota}_{\tau_{l}}+ \tilde{\alpha}(\tau_l,V^{N,\iota}_{\tau_{l}},\frac{1}{B} \sum^B_{i=1} \delta_{\tilde{X}^i_{{j+1}}}) \Delta \tau + \Delta W^\iota_{\tau_l}, \ \ V^{N,\iota}_{\tau_0}=0, \ \iota \in \lbrace 1,2,...,B \rbrace.
      \end{align} 
      where the function $\tilde{\alpha}$ is defined in \eqref{app_r}. 
      Then $\tilde{A}^{N,B}_j\bbP(X_{{j+1}}|\mcal{Y}_{j}) =\frac{1}{B} \sum^B_{\iota=1} \delta_{V^{N,\iota}_{T}}$.
    \end{itemize}}
\ENDFOR
\end{algorithmic}
\end{algorithm}
We make a few remarks and implementation notes concerning Algorithm \ref{algorithm_EnSBF}:
\begin{enumerate}
    \item The proposed algorithm is built on the Schr{\"o}dinger bridge generative model discussed in Section 2, which is training-free and derivative-free. The main computational cost comes from propagating the auxiliary particles through the discretized dynamics \eqref{discrete_drift_1} and the size ensemble particles to be generated. Since these particle trajectories can be simulated independently, this step is naturally parallelizable and can be implemented efficiently on modern parallel hardware.
    
    \item In the algorithm, in general we could start with $B$ initial particles $\mathbb{P}(X_0|\mcal{Y}_0)$ and proceed by generating $M$ particles for the posterior step. In the current algorithm design, we take $M=B$. In fact, once the data assimilation passes burn-in time, one could produce arbitrarily large particle cloud for the posterior distribution, i.e. $M>B$.

    \item The algorithm does not introduce surrogate model approximate error as in \cite{feng1} where in the update step, such error is introduced in the approximation for the score function of the posterior. In Algorithm \ref{algorithm_EnSBF} though, the mapping at the analysis step is approximated via: 
    $\tilde{A}_j \tilde{\mu}_{j+1} \approx \tilde{A}^{N,B}_j \tilde{\mu}_{j+1} =\frac{1}{B}\sum^B_{\iota=1} \delta_{V^{N,\iota}_{T}}$
    where errors are introduced only in the Euler–Maruyama discretization of the SDE \eqref{scheme_sde}, the ensemble average approximation of the expectation in the drift term \eqref{exact_drift_gtilde} and the ensemble approximation in the forecast and posterior distributions. 
    
    \item Importantly, the empirical approximation in \eqref{app_r} may suffer from numerical instability in implementation. The exponential weights can overflow or underflow when the dimension is large, or when $t$ is close to $T$, due to the factor $(T-t)^{-1}$ in the exponent. In this case, we can improve numerical stability by using a log-sum-exp implementation. More precisely, one subtracts the same constant, typically the maximum log-weight, from all logarithmic weights before exponentiation. Since this common shift appears in both the numerator and the denominator, it does not change the value of the drift mathematically.

    In more challenging regimes, however, the weights may remain highly concentrated even after log-sum-exp stabilization. In this case, one may further temper or clip the logarithmic weights to prevent degeneracy. Such a modification is inspired by logit tempering and can improve numerical robustness in practice. However, such approach may introduce bias into the resulting drift approximation.  

    \item In the scheme above, all auxiliary particles are initialized at the origin during
the analysis step. In fact, it can be shown that one could construct $\bP \in \mcal{P}(\Omega)$ induced by a stochastic process such that it interpolates between $\delta_{a}, a \in \bR^d$ and the target distribution $\mu$ (see also \cite{sb10}). 
The drift term then takes the following form: 
\begin{align}
    \tilde{\alpha}(t,x,\tilde{\mu})&:=\frac{\E_{Z \sim \tilde{\mu}}[(Z-x) \exp(-\frac{1}{2(T-t)}|Z-x|^2+\frac{1}{2 T} |Z-a|^2)\tilde{g}(Z)]}{(T-t)\E_{Z \sim \tilde{\mu}}[\exp(-\frac{1}{2(T-t)}|Z-x|^2+\frac{1}{2 T} |Z-a|^2)\tilde{g}(Z)]}
\end{align}
From an implementation perspective, the benefit of this form is that one can select the starting position of initial ensemble other than 0. Such a setup will be useful when the samples from distribution $\tilde{\mu}$ are far from zero in which case one can choose $a$ to be the empirical mean of the ensemble associated with $\tilde{\mu}$. 

\item Moreover, we note that the quality of the numerical approximation \eqref{app_r} for equation \eqref{exact_drift_gtilde} depends on the quality of the samples coming from $\tilde{\mu}_{j+1}$, and the function $\tilde{g}_j$. In the Gaussian observation noise case, the likelihood takes the form
\begin{align}\label{likelihood_cm1}
   \tilde{g}_j(z) \propto \exp \bsparath{-\frac{1}{2}(g(z)-Y_{j+1})^T\Sigma^{-1} (g(z)-Y_{j+1}))}
\end{align}
where $Y_{j+1}$ are the state of the observational process. 
If data samples from $\tilde{\mu}_{j+1}$ are too far away from the mode of function $\tilde{g}_j$, the exponential decay behavior of \eqref{likelihood_cm1} will render the Monte Carlo integration in \eqref{exact_drift_gtilde} inefficient.  As such, one could resort to other simulation techniques which could potentially improve the accuracy of approximation for the drift term. As a starter, inspired by importance sampling technique, one can perform a change of measure for potential performance improvement. 

To proceed, instead of using the $\lbrace \tilde{X}^i_j \rbrace^B_{i=1}$ particles from the prediction step, one can use an alternative proposal which also incorporates the new observation $Y_{j+1}$: 
\begin{align}
    \tilde{X}^i_{j+1} \sim \mathbb{Q}(\cdot|X^i_j, Y_{j+1}).
\end{align}
That is, at the $j$th filtering step, one has available the ensemble particles $\lbrace X^i_{j}\rbrace^B_{i=1}$ and also the observation $Y_{j+1}$. Thus, samples are drawn from a proposal distribution $\mathbb{Q}(\cdot|X^i_j, Y_{j+1})$. Denoting $h_a(z,x):=\exp(-\frac{1}{2(T-t)}|z-x|^2+\frac{1}{2 T} |z-a|^2)$, and noting that $\tilde{\mu}_{j+1} = \bbP(X_{j+1}|Y_{j})$ then after change of measure, the expectation can be replaced with  
\begin{align}\label{change_of_measure_eq}
\E_{Z \sim \tilde{\mu}_{j+1}}[h_a(Z,x)\tilde{g}_j(Z)]& = \E_{Z \sim \mathbb{Q}(x|X_j, Y_{j+1}), X_j \sim \mu_j}[h_a(Z,x) \frac{\bbP(Y_{j+1}|Z)\bbP(Z|X_{j})}{\mathbb{Q}(Z|X_j, Y_{j+1})}]. 
\end{align}
 We note that in \eqref{change_of_measure_eq}, the choice of $\mathbb{Q}(\cdot |X_j, Y_{j+1}):= \bbP(\cdot |X_j)$ will recover the original algorithm. Thus defining $w^l_j:= \frac{\bbP(Y_{j+1}|\tilde{X}^l_{j+1})\bbP(\tilde{X}^l_{j+1}|X^l_{j})}{\mathbb{Q}(\tilde{X}^l_{j+1}|X^l_j, Y_{j+1})}$ the new function $\alpha$ at each filtering stage can be  expressed in the following form: 
\begin{align}
    \bar{\alpha}_j(t,x) := \frac{\sum^B_{l=1} (\tilde{X}^l_{j+1} -x) h_a(\tilde{X}^l_{j+1},x) w^l_j}{(T-t)\sum^B_{l=1} h_a(\tilde{X}^l_{j+1},x) w^l_j} \label{change_of_measure_app}
\end{align} 
\item For the convergence analysis of EnSBF, the main theoretical challenge lies in establishing convergence of the one-step static Schrödinger bridge sampler corresponding to Algorithm~\ref{algorithm_drift}. Once this result is available, the analysis of the full filtering procedure can be developed through a recursive argument similar to that used for particle filters. Existing convergence results for Schrödinger bridge samplers generally assume that the target distribution is available in analytic form \cite{jiao1,wangzhang}. In the present setting, however, the target distribution is accessible only through samples, and the corresponding numerical convergence result requires further development. A forthcoming work by the author addresses this sample-based setting, while a complete convergence analysis of the full EnSBF framework is left for future research. 

\end{enumerate}

\subsection{Further discussion on EnSBF}
In this section, we discuss several theoretical and computational aspects of EnSBF. As we will show in Section 3.3, the EnSBF achieves competitive performance compared with other nonlinear filters in low/moderately high-dimensional settings. Moreover, Section 2.2 shows that the reverse SDE arising from score-based diffusion models is a special case of the controlled SDE associated with the Schr{\"o}dinger bridge problem. This establishes a natural link between the EnSF and EnSBF methods. The figure below summarizes the design process and highlights the connection between the two approaches. 

\begin{figure}[h]
\centering
\begin{tikzpicture}[
    bluebox/.style={rectangle, draw=black, thick, fill=blue!10, 
                    minimum width=3cm, minimum height=1.5cm, align=center},
    pinkbox/.style={rectangle, draw=black, thick, fill=pink!30, 
                    minimum width=2.5cm, minimum height=1.2cm, align=center},
    line/.style={thick, black},
    arrow/.style={-{Stealth[scale=1.2]}, thick, black!70}
]

\node[bluebox] (left) {SBP};

\draw[line] (left.east) -- ++(1,0) coordinate (split-point);

\draw[line] (split-point) -- ++(0,1.2);  
\draw[line] (split-point) -- ++(0,-1.2); 

\draw[line] (split-point) ++(0,1.2) -- ++(1,0) coordinate (up-end);
\draw[line] (split-point) ++(0,-1.2) -- ++(1,0) coordinate (down-end);

\node[bluebox, anchor=west] (up-box) at (up-end) 
    {$b=0$ \\ $\sigma$ constant};
    
\node[bluebox, anchor=west] (down-box) at (down-end) 
    {$b=-\tilde{b}_{1-t}$ \\ $\sigma_t=\tilde{\sigma}_{1-t}$};

\draw[line] (up-box.east) -- ++(1,0) coordinate (up-right);
\draw[line] (down-box.east) -- ++(1,0) coordinate (down-right);

\node[bluebox, anchor=west] (up-final) at (up-right) 
    {$f_0=\delta_0$ \\ $g_1(y)= \frac{p_{\text{data}}(y)}{\phi_\sigma(y)}$};
    
\node[bluebox, anchor=west] (down-final) at (down-right) 
    {$f_0=1$ \\ $g_1(y)=p_{\text{data}}(y) e^{I_{0,1}}$};

\draw[arrow] (up-final.east) -- ++(0.75,0) 
    node[pinkbox, anchor=west] (pink1) {EnSBF};

\draw[arrow] (down-final.east) -- ++(0.75,0) 
    node[pinkbox, anchor=west] (pink2) {EnSF};

\fill[black] (split-point) circle (2.5pt);

\end{tikzpicture}
\caption{Flow chart of scheme design: starting from the SBP, the split node separates into the top part and bottom part. The first two vertical boxes correspond to the reference SDE used to formulate the SBP; the next two blue vertical boxes correspond to the solution strategies, i.e. the choices of $f_0,g_1$. The dichotomy results in the EnSBF and EnSF filters accordingly.}
\end{figure}

Both EnSBF and EnSF can be interpreted through the Schr{\"o}dinger bridge framework, but they correspond to different choices of reference dynamics. The EnSBF uses $b_t=0$, $\sigma_t=\sigma$, whereas the EnSF is connected to the time-reversed reference coefficients $b_t=-\tilde b_{1-t}$ and $\sigma_t=\tilde\sigma_{1-t}$, where $\tilde b_t$ and $\tilde\sigma_t$ are defined in \eqref{ensf_eq1}--\eqref{ensf_eq2}. These different reference choices lead to different reverse, or generative, SDEs and therefore to different numerical
simulation schemes. Motivated by their connection through the Schr{\"o}dinger bridge formulation, these two approaches complement each other in different dimensional regimes.

Although the EnSBF algorithm is parallelizable, its arithmetic complexity remains $\mathcal{O}(NB^2d)$, where $N$ denotes the number of artificial-time discretization steps, $B$ is the ensemble size, and $d$ is the state dimension. The quadratic dependence on $B$ arises because, at each artificial-time step, every generated posterior particle interacts with all forecast particles when evaluating the empirical drift. Consequently, the computational cost can become substantial as either the ensemble size or the state dimension increases. As demonstrated in the numerical experiments of the next section, this cost remains manageable for low- and moderate-dimensional problems, since EnSBF typically performs well with fewer than $100$ artificial-time steps and ensemble sizes below $2^7$. In high-dimensional settings, however, global likelihood weights can become severely degenerate. The results in Figure~\ref{fig:comparison_ensemble_temperature} indicate that increasing the ensemble size may then be necessary to maintain satisfactory accuracy, which further amplifies the computational burden. 

To mitigate the deterioration caused by weight degeneracy in high dimensions, we introduce a spatial localization strategy inspired by local particle filtering \cite{van_handel}. Rather than updating every coordinate of a generated particle using the same global weight, we construct coordinate-dependent weights from observations and state components within a prescribed local neighborhood. Consequently, each coordinate is updated using information that is locally relevant to that coordinate, thereby reducing the effective dimension of the weighting problem and alleviating global weight degeneracy. That is, instead of computing the weights globally for each coordinate as
    \begin{align}
        \pi^m(t,x):= \frac{w^m(t,x)}{\sum w^m(t,x)}, \quad w^m(t,x) = \exp \Big( \frac{|X^m-a|^2}{2}-\frac{|X^m-x|^2}{2(1-t)} - \frac{|g(X^m)-y^{\text{obs}}|^2}{2 \sigma^2_{\text{obs}}} \Big),
    \end{align}
    we consider a local weight construction. For each coordinate $j \in \lbrace 1, ..., d \rbrace$, define a collection of the neighborhood coordinates $\mcal{O}_j := \lbrace j-r, ...j,..., j+r \rbrace$. Then for each particle $X^m$, $m \in \lbrace 1, ..., B \rbrace$, we define the local error: 
    \begin{align}
        E^m_j := \frac{1}{|\mcal{O}_j|} \sum_{k \in \mcal{O}_j} \frac{(g_k(X^m)-y_k^{\text{obs}})^2}{\sigma^2_{\text{obs}}}. 
    \end{align}
We define $L^j_m : = \exp(-\frac{1}{2} E^m_j)$. In a similar fashion we let: 
\begin{align}
    K^j_m=\exp \left( \frac{1}{2|\mcal{A}_j|} \sum_{k \in \mcal{A}_j} |X^m_k -a_k|^2 - \frac{1}{2(1-t)|\mcal{A}_j|} \sum_{k \in \mcal{A}_j}|X^m_k -x_k|^2 \right)
\end{align}
where $\mcal{A}_j := \lbrace j-r', ...j,..., j+r' \rbrace$. Then we have: 
\begin{align}
    \pi^j_m(t,x) = \frac{w_m^j(t,x)}{\sum^B_{l=1} w_l^j(t,x)}, \qquad w^j_m:= K^j_m L^j_m.
\end{align}
Based on the discussion, the drift in the EnSBF becomes: 
\begin{align}
    \alpha^j(t,x) := \frac{\sum^B_{m=1} \pi^{j}_m (t,x) X^m_j -x_j}{1-t}.
\end{align}
This localization strategy is motivated by the observation that, in many spatially extended dynamical systems, the interaction between nearby state coordinates is substantially stronger than that between distant coordinates. Consequently, the update of a given coordinate can often be constructed primarily from information contained in a local neighborhood without discarding the dominant dependence structure of the system \cite{van_handel}. Although this approximation introduces localization, or bias by truncating long-range dependencies which is further combined with temperature, it alleviates the deterioration caused by global weight degeneracy in high dimensions. The current localization method is designed purely for the case where the states are fully observed. Extensions to sparse or partially observed systems are left for future work.

The leading computational complexity remains $\mathcal{O}(NB^2d)$, since the localized coordinate averages can be computed once and reused throughout each drift evaluation. Nevertheless, our numerical experiments indicate that the localized EnSBF performs well in high-dimensional settings with both the number of artificial-time steps $N$ and the ensemble size $B$ remaining on the order of $20$, even when the state dimension is approximately $10^3$. Detailed numerical results are presented in the next section.

\subsection{Numerical examples for EnSBF}
In this section, we provide numerical examples for the designed algorithms.
\begin{enumerate}
    \item Example 1 is a one-dimensional example demonstrating that, in low dimensions with only mild nonlinearity, the EnSBF achieves performance comparable to the Ensemble Kalman Filter and the Particle Filter. We also study the convergence behavior of the method as the number of temporal discretization steps and the ensemble size used for posterior sample generation are increased. 
    
    \item Example 2 considers a one-dimensional double-well potential problem. This example shows that the EnSBF can outperform both benchmark methods when shocks are present in the signal dynamics. 
    
    \item Example 3 provides a more detailed comparison between the Particle Filter and the EnSBF for one-step data assimilation setting. We design a one-step filtering problem in which the prior, likelihood, and posterior densities are all known analytically. The numerical results show that, even when a large ensemble size is used, the EnSBF produces a less impoverished equal weight ensemble and is competitive with PF, while neither method uniformly dominates across all reported metrics.

    \item Example 4 considers the Lorenz 96 model, where the performance of the EnSBF is tested across dimensions ranging from $4$ to $100$ under linear observations. The performance is compared against the EnKF, PF, and EnSF. The results show that the EnSBF performs reasonably well in lower dimensions, while its predictions tend to deteriorate in high dimensions, exhibiting larger fluctuations around the true signal. The EnKF shows generally strong performance, likely due to the relatively mild nonlinearity of the underlying dynamics. Finally, when the signal goes through nonlinear transformation in the observation process, the performance of EnSBF is observed to deteriorate quickly once it goes above 40.
    
    \item Example 5 considers Lorenz-96 model and the Kuramoto-Sivashinsky (KS) PDE in high dimension where coordinate localization technique is used in conjunction with EnSBF. In this case, model structural bias is deliberately introduced so that EnSBF can more effectively perform state estimation. 
\end{enumerate}

\subsubsection{Example 1: the sine function.}
In this example, we consider the following simple 1-D dynamical system: 
\begin{align}
X_{j+1} &= \alpha \sin(X_j) + \sigma \xi_j, \ \ \xi_j \sim \mcal{N}(0,1) \label{eg_d1} \\ 
Y_{j+1} &=X_{j+1}+ \gamma  w_{j+1}, \ \ w_{j+1} \sim \mcal{N}(0,1) \label{eg_d2} 
\end{align}
In this example, we set $\alpha=2.5$, $\sigma=0.2$ and $\gamma=1.0$. We compare the performance of three nonlinear filtering methods: the Ensemble Kalman Filter (EnKF), the Particle Filter (PF), and the Ensemble \sbb Filter (EnSBF). For all methods, the ensemble size is set to 500. The test is performed independently for 30 times and an average RMSE is computed.  For the EnSBF, the number of temporal discretization steps used to simulate the auxiliary SDE in each analysis step is set to $100$.  Figure \ref{fig:1dcompare} a shows that all three filters exhibit a decaying trend in the smoothed RMSE which is computed by using a moving window of size 50. As expected, the EnKF slightly underperforms compared to PF and EnSBF due to the nonlinearity in the dynamics \eqref{eg_d1}. Meanwhile, since this is a one-dimensional example and a relatively large ensemble size is used, the PF and the EnSBF demonstrate similar performance. 

To study the convergence behavior of the EnSBF over 800 filtering steps, we first fix the ensemble size to be 200 and change the number of Euler discretization $N$. For each value of $N$, we repeat the experiment $50$ times and report the averaged smoothed RMSE. As shown in Figure \ref{fig:1dcompare} b, the smoothed RMSE decreases as the number of discretization steps $N$ in \eqref{discrete_drift_1} increases. This behavior is expected, since the quality of the generated particles depends on the accuracy of the Euler approximation of the diffusion SDE \eqref{scheme_sde}.  

We then fix the number of temporal discretization steps $N$ and vary the ensemble size $B$. Again, each experiment is repeated $50$ times, and the averaged smoothed RMSE is reported. Figure \ref{fig:1dcompare} c shows that the smoothed RMSE also decreases as the ensemble size increases. This is consistent with the fact that the approximation of the drift term depends on the number of particles used to approximate the expectations in \eqref{app_r}.  
\begin{figure}[h]
    \centering 
    \includegraphics[width=0.9\textwidth]{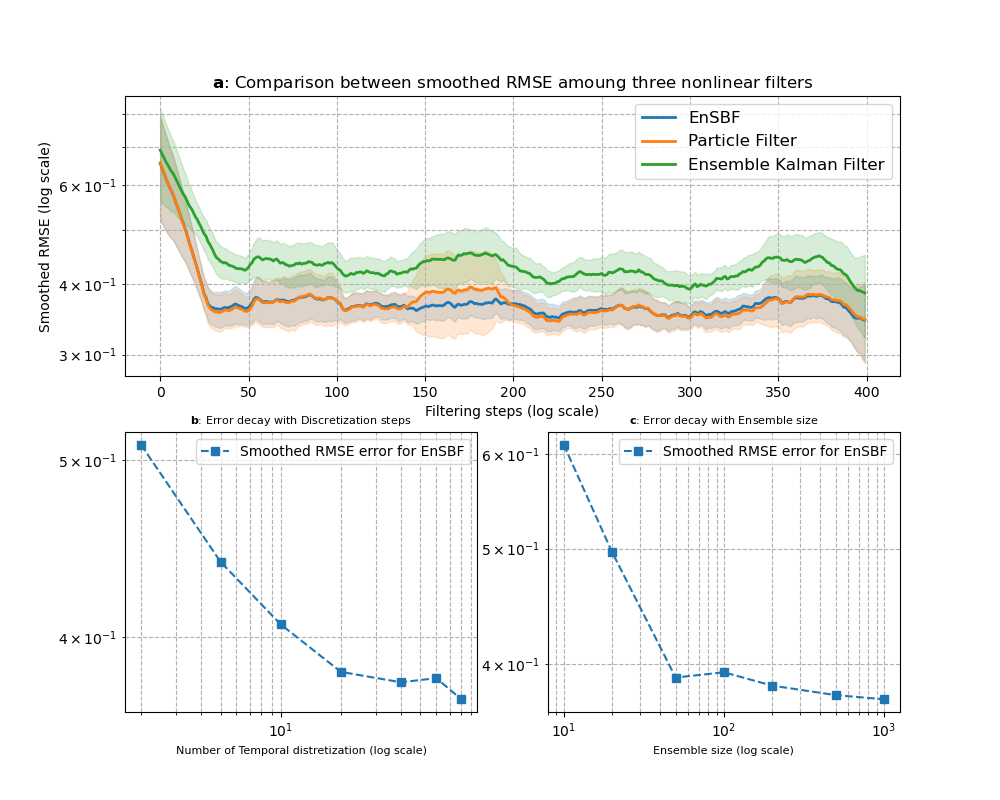} 
    \caption{a: Comparison between the smoothed mean RMSE among the three nonlinear filters with a moving average window of size 50. b: Error decay with respect to the number of temporal discretization $N$ in the diffusion SDE, holding the ensemble size $B$ fixed. c: Error decay with respect to the size of the particle ensemble, holding $N$ fixed} 
    \label{fig:1dcompare} 
\end{figure}

\subsubsection{Example 2. The 1-D potential well }
The second example considers the double potential-well where the drift term is given by the function $-4x(x^2-1)$. 
The dynamics is given as follows: 
\begin{align}
S_{n+1}&=S_n-(4\cdot S_n(S^2_n-1))\Delta t+\beta \sqrt{\Delta t} \omega_n \label{pw1} \\ 
M_{n+1}&=S_{n+1} +\gamma_{n+1} \label{pw2}
\end{align}
Here, the signal particles are propagated according to \eqref{pw1}, while the observations are generated from \eqref{pw2}. To test the responsiveness of the filters to abrupt transitions, the true state is manually switched between the two potential wells, $1$ and $-1$, every $40$ time steps. Note that it is easy for the nonlinear filter to track the stationary state while hard to effectively respond to the sudden state change.  

In this example, we compare the filtering results obtained from several nonlinear filtering methods: the Ensemble Kalman Filter (EnKF), the Particle Filter (PF), the Ensemble Schr{\"o}dinger Bridge Filter (EnSBF), and a variant of the EnSBF based on the change-of-measure, denoted by EnSBF-I. We set $\gamma_{n}\sim\mathcal N(0,0.1)$ for all time steps and choose
$\Delta t=0.1$. We comment on the test results presented in Figure \ref{fig:1-d pw} with the exact states shown in blue, the EnSBF shown in yellow, EnSBF\_I in green, PF in red and ensemble KF in purple. 

\begin{enumerate}[i.)]
    \item In the case shown in Figure \ref{fig:sub1}, where a large ensemble size is used $(B=1000)$ and the signal noise level in \eqref{pw1} is set to  $\beta=0.3$, all nonlinear filters are able to track the signal reasonably well. However, the EnSBF provides more accurate state tracking compared with the benchmark methods.

    \item In the case shown in Figure \ref{fig:sub2}, the same ensemble size $(B=1000)$ is used, but the signal noise level is reduced to $\beta=0.2$. In this regime, neither the EnKF nor the PF performs well. This phenomenon can be explained by the fact that particle filters rely heavily on the quality and coverage of the ensemble particles for distributional approximation. When a sudden state transition occurs, the drift of the double-well model still pulls particles toward the current potential well. As a result, the particle cloud may remain highly concentrated near the
    bottom of one well, leading to a small ensemble variance and insufficient particle coverage in the tail region where the transition occurs. The EnKF also struggles in this example because the dynamics are highly nonlinear and the posterior distribution can be strongly non-Gaussian. Since the EnKF relies on a Gaussian-type approximation of the filtering distribution, it may fail to capture the multimodal or non-Gaussian structure induced by the double-well dynamics. In contrast, both EnSBF-based approaches remain effective in this setting.
    
    \item In the case shown in Figure \ref{fig:sub3}, the ensemble size is reduced to $B=20$, while the signal noise level is kept at $\beta=0.3$. The two Schr{\"o}dinger bridge-based filters and the EnKF are still able to track the signal, with the EnSBF showing slightly better performance. Since the EnKF relies on a Gaussian-type approximation of the filtering distribution, reducing the ensemble size does not significantly degrade its performance in this relatively simple setting. In contrast, the PF represents the filtering distribution solely through weighted particles. Therefore, when the ensemble size is small, the PF fails to produce an accurate
    posterior approximation.

    Overall, this example shows that the proposed EnSBF method does not require a large number of particles to capture the relevant statistical features of the filtering distribution between observation times. This represents a significant improvement over the PF, which typically requires a much larger particle ensemble to achieve stable performance (see Figure \ref{fig:sub1}). To understand why EnSBF succeeds in this setting, note that its mechanism differs fundamentally from that of the particle filter. EnSBF evolves the ensemble through a stochastic bridge whose drift is guided by the likelihood information. Consequently, the particles are continuously transported toward regions favored by the new observation. This allows an ensemble initially concentrated near one potential well to move rapidly toward the other well when the observation indicates a state transition.

    \item Lastly  in (\ref{fig:sub4}), We use an ensemble size of 20 and reduce the noise in the signal to $\beta=0.1$. In this case,  the performance of the original EnSBF starts to deteriorate. Both the EnKF and PF fail to perform while the EnSBF-based filters still demonstrate some tracking ability. 
\end{enumerate}

\begin{figure}[htbp]
    \centering
    \begin{subfigure}[b]{0.5\textwidth}
        \includegraphics[width=\textwidth]{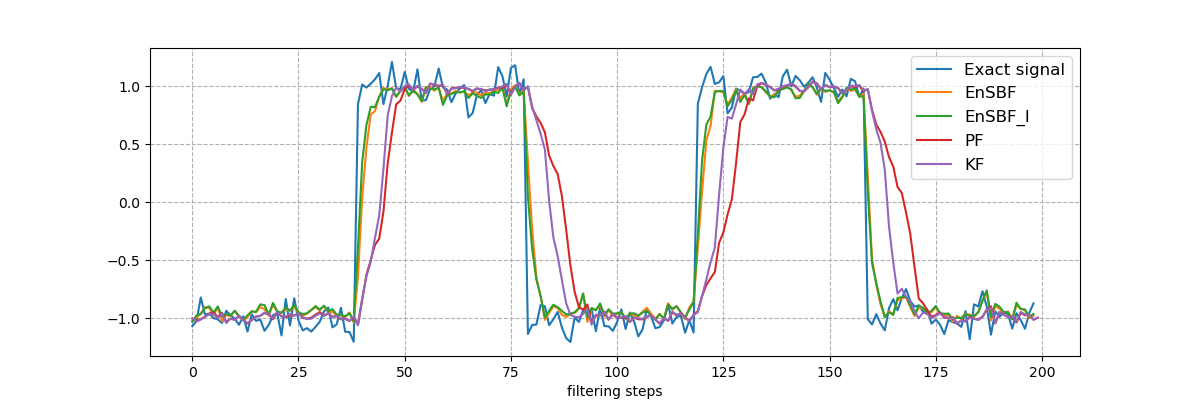}
        \caption{Ensemble size $1000$, $\beta=0.3$}
        \label{fig:sub1}
    \end{subfigure}
    \hspace{-0.5cm}
    \begin{subfigure}[b]{0.5\textwidth}
        \includegraphics[width=\textwidth]{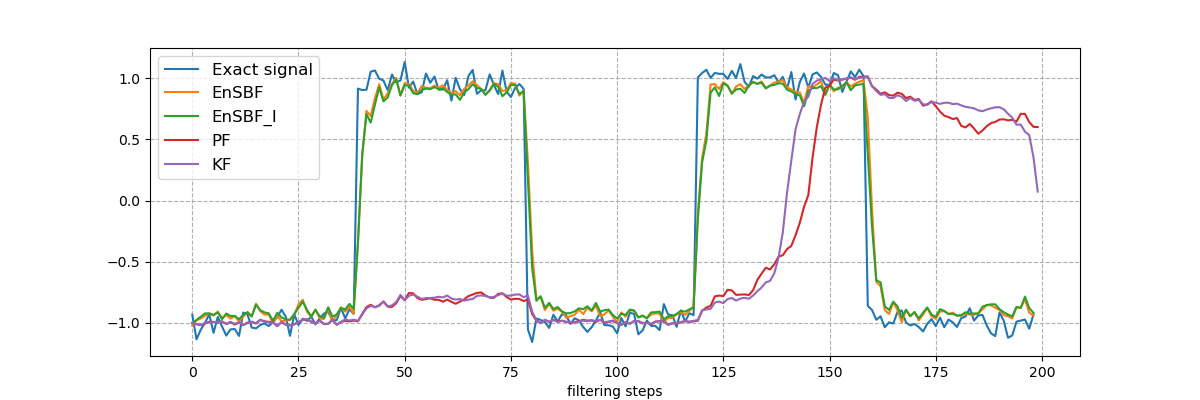}
        \caption{Ensemble size $1000$, $\beta=0.2$}
        \label{fig:sub2}
    \end{subfigure}
    \vspace{0.1cm} 
    \begin{subfigure}[b]{0.5\textwidth}
        \includegraphics[width=\textwidth]{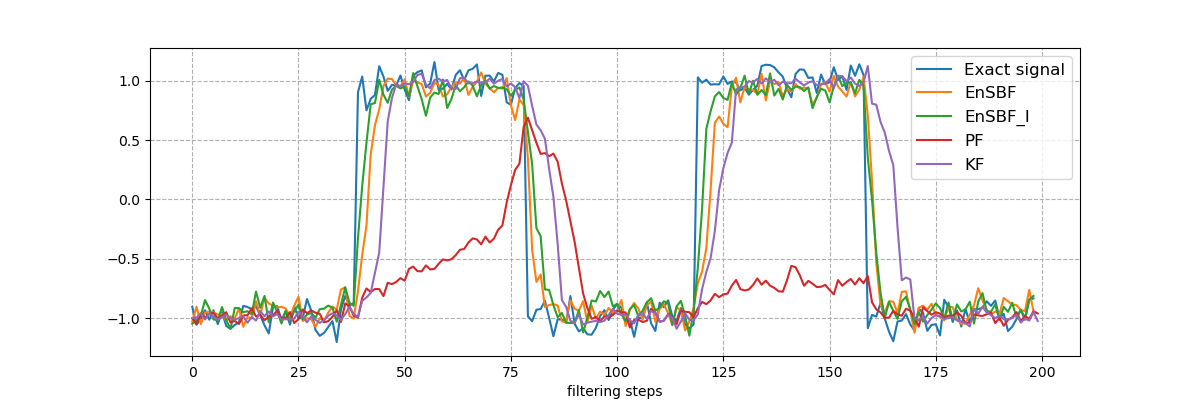}
        \caption{Ensemble size $20$, $\beta=0.3$}
        \label{fig:sub3}
    \end{subfigure}
    \hspace{-0.5cm}
    \begin{subfigure}[b]{0.5\textwidth}
        \includegraphics[width=\textwidth]{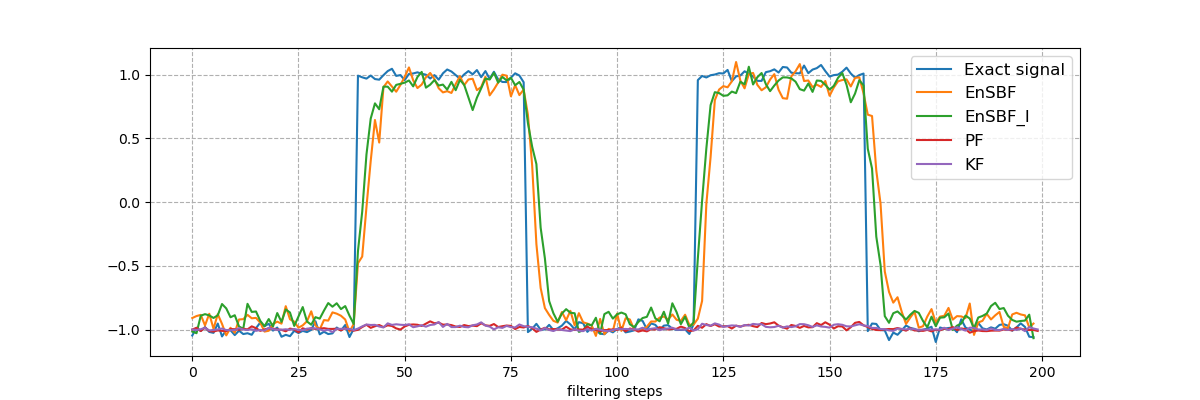}
        \caption{Ensemble size $20$, $\beta=0.1$}
        \label{fig:sub4}
    \end{subfigure}
    \caption{Comparison of state tracking among nonlinear filters  }
    \label{fig:1-d pw}
\end{figure}

\subsubsection{Example 3, Further comparison between EnSBF and PF}
As a third numerical experiment, we consider a one-step filtering problem in which both the prior density and the likelihood are available in closed form, allowing the corresponding posterior density to be evaluated analytically. For each test case, we draw $2{,}000$ samples from the prior and generate $2{,}000$ posterior samples using both EnSBF and the particle filter. The EnSBF stochastic differential equation is discretized using $500$ time steps. The three test cases are specified as follows.
\subsubsection*{Multimodal Gaussian Mixture Prior}
 More specifically, let the prior/likelihood be given as Gaussian mixture/Gaussian: 
\begin{align}
    \mathbb{P}(X_{j+1}=x|\mcal{Y}_j) &= \frac{1}{4} \sum^4_{k=1} \mcal{N}(x| \mu_k, \sigma^2 I) && \textit{Prior density} \\
    \mathbb{P}(Y_{j+1}|X_{j+1}=x) &\propto \mcal{N}(x|\mu_0, \epsilon^2 I) && \textit{Likelihood density} \\
    \mathbb{P}(X_{j+1}|\mcal{Y}_{j+1})& = \sum^4_{k=1} w_{k}\mcal{N}(x| m_k,\Sigma_k) && \textit{Posterior density} 
\end{align}
where $m_k=\frac{\epsilon^2 \mu_k + \sigma^2 \mu_0}{\sigma^2 + \epsilon^2}, \Sigma_k=\frac{\sigma^2 \epsilon^2}{\sigma^2+\epsilon^2} I$ and $ w_k=\frac{\exp\left(-\frac{|\mu_k-\mu_0|^2}{2(\sigma^2+\epsilon^2)}\right)}{\sum_{\ell=1}^4\exp\left(-\frac{|\mu_\ell-\mu_0|^2}{2(\sigma^2+\epsilon^2)}\right)}$. For this example, we choose $\mu_0=(1.2,0.0), \mu_1=(1.5,1.0), \mu_2=(1.0,-1.0), \mu_3=(-1.5,1.0), \mu_4=(-1.0,-1.0)$, $\sigma=0.2$ and $\epsilon=0.25$.

\subsubsection*{Quadratic Nonlinear observation}
For the second example, we again consider a two-dimensional example where with the observational data given to be $y_{\text{obs}}=(1,0)$, $\sigma_{\text{obs}}=0.25$, the observation function is given by $h(x_1,x_2) =(x^2_1, x_2)^T$ and the density function is given by $p_{\textbf{prior}} \propto \exp( -\frac{x^2_1 + x^2_2}{2})$. Hence, the likelihood and the posterior take the following form: 
\begin{align}
    p_{\text{likelihood}} & \propto \exp \left( -\frac{(x^2_1-1)^2+x^2_2}{2 \sigma^2_{obs}} \right) \\
    p_{\text{posterior}} & \propto \exp(-\frac{x^2_1+x^2_2}{2} - \frac{(x^2_1-1)^2+x^2_2}{2 \sigma^2_{obs}}).
\end{align}

\subsubsection*{Radial Nonlinear observation}
For the third example, we take the prior distribution to be a centered Gaussian with precision matrix $\Sigma^{-1}$, given by
$$
p_{\text{prior}}(x)\propto \exp \left(-\frac{1}{2}x^\top \Sigma^{-1} x\right),
\qquad
\Sigma=
\begin{pmatrix}
0.5 & -0.4\\
-0.4 & 0.5
\end{pmatrix}.
$$
The observed value and the noise are $y_{\text{obs}}=1.5$ and $\sigma_{\text{obs}}=0.1$. The nonlinear observation function is given by $h(x_1, x_2)= \sqrt{(x_1-1)^2 + (x_2-1)^2}$. Hence we have: 
\begin{align}
    p_{\text{likelihood}} & \propto \exp \left( - \frac{\left( \sqrt{(x_1-1)^2 + (x_2-1)^2} -3/2 \right)^2}{2 \sigma^2_{\text{obs}}} \right)\\ 
    p_{\text{posterior}} & \propto  \exp \left( -\frac{1}{2}x^\top \Sigma^{-1} x  - \frac{\left( \sqrt{(x_1-1)^2 + (x_2-1)^2} -3/2 \right)^2}{2 \sigma^2_{\text{obs}}} \right).
\end{align}
\subsubsection*{Test result}
The test results are summarized in Table \ref{tab:sbf-pf-one-step} and Figure \ref{fig:designed_compare}. We make the following observations: the posterior ensemble distributions obtained from both filtering methods can capture the geometry of the true posterior and the reported statistics are close. However, judging from the particle filter for all three testing examples, the particle filter suffers substantial weight degeneracy as indicated by the small effective sample size and the limited number of unique particles remaining after resampling. This sample impoverishment is also visible in the particle clouds in the last column of figures in Figure \ref{fig:designed_compare}, which contain more pronounced clustering and repeated sample locations. Overall, these results suggest that SBF is less affected by post-resampling sample impoverishment in settings with strongly concentrated likelihood weights.

\begin{table}[htbp]
\centering
\caption{Comparison of the one-step posterior approximations obtained using SBF and the particle filter. All reported values are rounded to two decimal places.}
\label{tab:sbf-pf-one-step}
\resizebox{\textwidth}{!}{
\begin{tabular}{llccccccc}
\toprule
Example & Method & Mean error & Covariance error & Observation RMSE & Sliced $W_2$ & Empirical $W_2$ & PF ESS & Unique ancestors \\
\midrule
Multimodal GMM
& SBF & 0.08 & 0.04 & 0.69 & 0.13 & 0.21 & -- & -- \\
& PF  & 0.09 & 0.03 & 0.68 & 0.14 & 0.23 & 52.91 & 282.00 \\
\midrule
Quadratic observation
& SBF & 0.12 & 0.02 & 0.35 & 0.23 & 0.30 & -- & -- \\
& PF  & 0.14 & 0.03 & 0.35 & 0.25 & 0.29 & 163.01 & 314.00 \\
\midrule
Radial observation
& SBF & 0.04 & 0.01 & 0.11 & 0.03 & 0.15 & -- & -- \\
& PF  & 0.05 & 0.01 & 0.10 & 0.04 & 0.11 & 627.79 & 678.00 \\
\bottomrule
\end{tabular}
}
\end{table}

\begin{figure}[h]
    \centering 
    \includegraphics[width=0.9\textwidth]{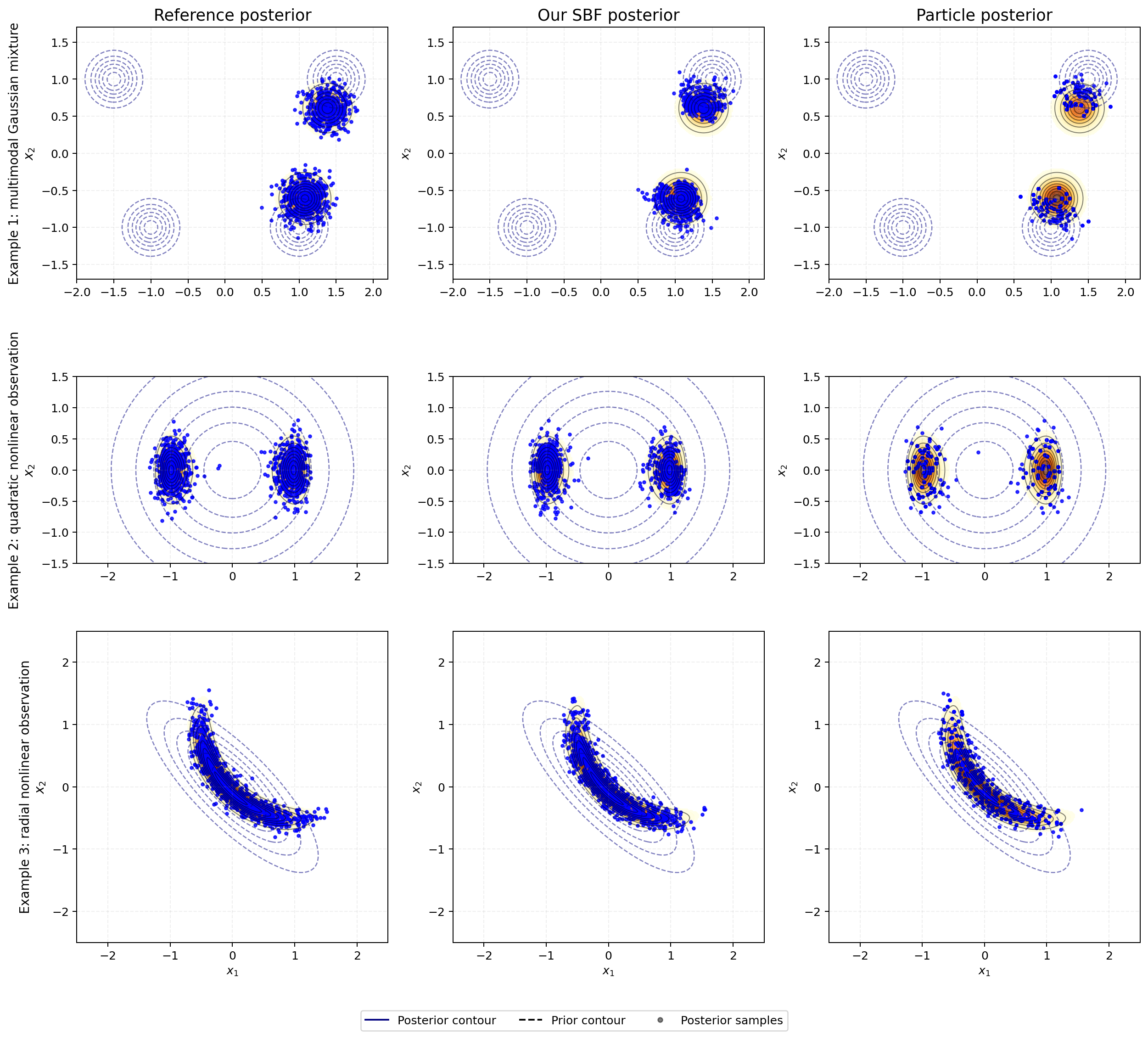} 
    \caption{Prior, likelihood and posterior density plots for the three one-step filtering examples.} 
    \label{fig:designed_compare} 
\end{figure}

\subsubsection{Example 4, the Lorenz-96 model}
For the fourth example, we consider the Lorenz-96 model. The goal is to study EnSBF's performance under different dimension regimes. The performance of the following four nonlinear filters is compared: PF, EnSBF, EnSF and EnKF. The overall observations are summarized as follows: 
\begin{itemize}
    \item  The EnSBF performs well overall, especially in low-dimensional regimes. However, its performance deteriorates as the dimension increases. In higher dimensions, the filtered estimates become noisier and tend to fluctuate more strongly around the true signal. 
    \item For the Lorenz--96 model, the EnKF also performs well. This is partly because the signal does not exhibit abrupt state transitions, unlike in the double-well potential example. As a result, the Gaussian-type approximation used by the EnKF is still able to capture the dominant structure of the filtering distribution reasonably well.
    \item The EnSF does not perform well in low-dimensional regimes, especially when the observation noise is large. This is mainly due to the model or structural error introduced by the score approximation in the algorithmic design. The EnSBF mitigates this limitation and demonstrates more competitive performance in such settings. On the other hand, when the underlying dimension becomes high, the EnSF outperforms the other filters, including the EnKF, consistent with the observations in \cite{feng6}.
\end{itemize}

We consider the dynamics in the following form
\begin{align}
    dX^i_t&=\Big( (X_t^{i+1}-X_t^{i-2})X_t^{i-1} +F \Big) dt + \sigma^i dW^i_t , \ i=1,2, ..., d, 
 d \geq 4 \label{eg2_1}\\ 
    Y_t &=\alpha X_t + \gamma_t, \ \gamma_t \sim \mcal{N}(0, \Gamma) \label{eg2_2}
\end{align}
Here, the indices are understood cyclically with $X^{-1}_t=X^{d-1}_t,X^0_t=X^d_t, X^{d+1}_t = X^1_t$. In this example, the nonlinear interaction term in the Lorenz--96 dynamics makes the filtering problem increasingly challenging as the dimension $d$ grows. We use an Euler scheme with $\delta t = 0.005$ for simulation of the dynamical system. In all of the experiments below, while changing $d$, we fix $\sigma=0.2$, across all dimensions, $F=8$, $\Gamma = 0.15^2 I$ and $\alpha=0.2$ with ensemble size $B=2^7$. In all of the tests performed, the smoothed RMSE is computed with a moving window with size 50 for visualization purposes. Tests were performed with 15 independent trials with average RMSE presented, shaded regions are pointwise $95\%$ confidence intervals for the mean. 

\subsubsection*{$\mathbf{d=4}$} 

Figure~\ref{fig:particle-4 state} shows that all four filters perform reasonably well, with EnSBF exhibiting the most stable and consistently accurate behavior. The particle filter performs comparatively worse, likely due to the relatively large observation noise and the limited ensemble size of $128$. EnSF displays a wider confidence band and requires a longer burn-in period, but its accuracy eventually becomes comparable to that of EnSBF and EnKF. EnKF also performs well, which is expected given the relatively large ensemble size and the linear observation structure.

\begin{figure}[htbp]
    \centering
    \begin{subfigure}[b]{0.49\textwidth}
        \includegraphics[width=\textwidth]{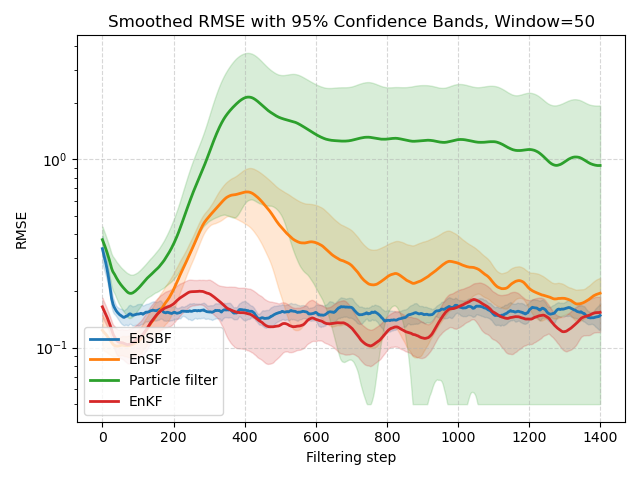}
        \caption{Smoothed RMSE}
        \label{fig:pf-sub1}
    \end{subfigure}
    \hspace{-0.6cm}
    \begin{subfigure}[b]{0.52\textwidth}
        \includegraphics[width=\textwidth]{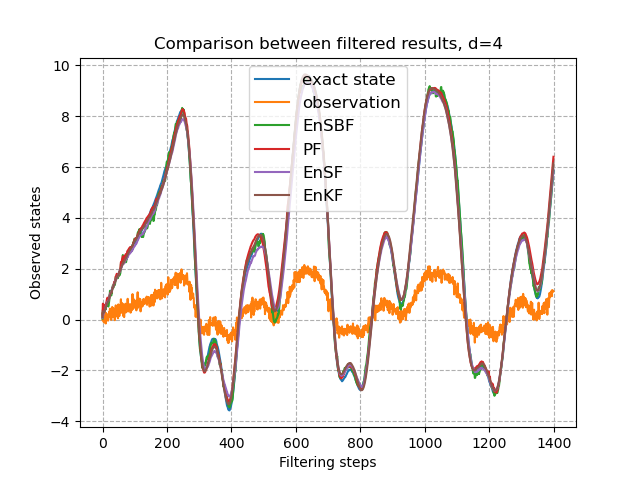}
        \caption{Representative coordinate $i=1$.}
        \label{fig:pf-sub2}
    \end{subfigure}
    \caption{Comparison of state tracking among nonlinear filters, Lorenz-96. $d=4$}
    \label{fig:particle-4 state}
\end{figure}

\subsubsection*{$\mathbf{d=20}$} 

In this case, the particle filter fails to track the signal, with its RMSE persistently oscillating between $4$ and $5$. This behavior is primarily attributable to the curse of dimensionality: as the state dimension increases, a substantially larger number of particles is required to adequately represent the filtering distribution. For this reason, the RMSE curve of the particle filter is omitted from the figure. The remaining three filters all exhibit acceptable performance, with EnKF achieving the lowest RMSE. EnSBF and EnSF attain comparable RMSE levels, although EnSF displays a noticeably wider confidence band. In dimension $20$, weight degeneracy begins to affect the likelihood term in EnSBF, reflecting a difficulty similar to that encountered by particle filters. For EnSF, structural approximation error in the posterior score introduces a relatively large bias. Together, these effects make EnSBF and EnSF less competitive than EnKF in terms of RMSE. It is also observed in Figure~\ref{fig:20dsub2} that, as the state dimension increases, EnSBF produces noisy predictions. This behavior is likely caused by the high variance of the estimated bridge drift. At each time step, the drift is constructed from likelihood-weighted averages over the entire ensemble. In higher dimensions, weight degeneracy and spurious long-range dependencies can cause these averages to be dominated by only a small number of particles, resulting in unstable state updates. This observation motivates the use of localization, which restricts each state update to relevant nearby particles or observations and may therefore reduce variance and stabilize the filtering results. 
  
\begin{figure}[htbp]
    \centering
    \begin{subfigure}[b]{0.49\textwidth}
        \includegraphics[width=\textwidth]{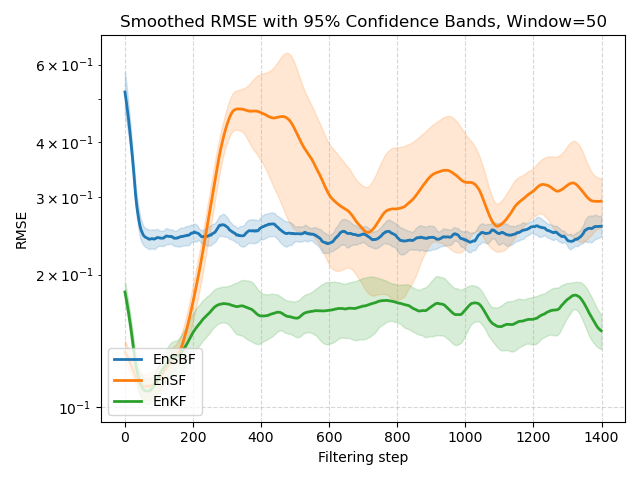}
        \caption{Smoothed RMSE}
        \label{fig:20dsub1}
    \end{subfigure}
    \hspace{-0.6cm}
    \begin{subfigure}[b]{0.52\textwidth}
        \includegraphics[width=\textwidth]{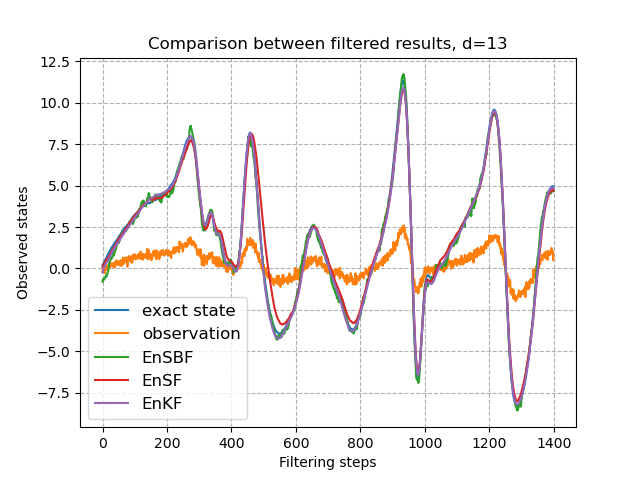}
        \caption{Representative coordinate $i=13$.}
        \label{fig:20dsub2}
    \end{subfigure}
    \caption{Comparison of state tracking among nonlinear filters, Lorenz-96. Total dimension $=20$}
    \label{fig:20-d state}
\end{figure}

\subsubsection*{$\mathbf{d=100}$} 
For a higher dimensional example, we increase the dimension to $d=100$. It is observed that EnSF begins to outperform EnSBF in the high-dimensional regime. This improved performance can be attributed to the fact that the likelihood score is already evaluated componentwise, independently of the mini-batch size. In particular, when the mini-batch size is set to one, each reverse particle is paired with a single forecast particle, so the prior score is also computed componentwise without an ensemble-weighted sum. This yields a numerically stable and efficient update. For EnSBF, the weight degeneracy issue in the likelihood term for drift estimation becomes more pronounced as dimension increases. The RMSE is observed to increase compared to the case when $d=20$ while $EnSF$ becomes smoother and remains at a similar level. The EnKF still performs better than both. This is due to the fact that the observational process is still linear and the ensemble size used is sufficiently large.  

\begin{figure}[htbp]
    \centering
    \begin{subfigure}[b]{0.49\textwidth}
        \includegraphics[width=\textwidth]{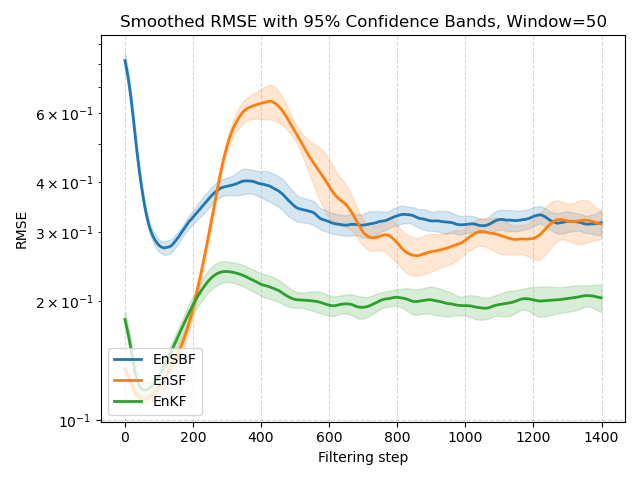}
        \caption{Smoothed RMSE}
        \label{fig:100dsub1}
    \end{subfigure}
    \hspace{-0.6cm}
    \begin{subfigure}[b]{0.52\textwidth}
        \includegraphics[width=\textwidth]{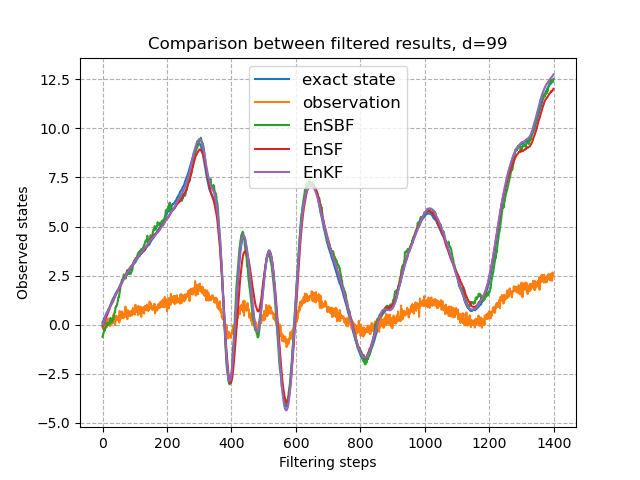}
        \caption{Representative coordinate $i=99$}
        \label{fig:100dsub2}
    \end{subfigure}
    \caption{Comparison of state tracking among nonlinear filters, Lorenz-96. Total dimension$=100.$}
    \label{fig:100-d state}
\end{figure}
\subsubsection*{Nonlinear observation}
To test the case where the observation is also nonlinear, we change \eqref{eg2_2} to the following: 
\begin{align}
     Y_t &= \text{arctan}(X_t) + \gamma_t, \ \gamma_t \sim \mcal{N}(0, \Gamma), \ \ \Gamma = 0.05^2 I_d \label{eg2_3}
\end{align}
We consider a higher-dimensional case with
$d=40$, using an ensemble size of $B=128$ and the same signal process setup as
before. The results are presented in Figure \ref{fig:tanh-40-d state}. We observe that the EnSBF achieves overall similar performance compared to EnSF in the current setting. The overall RMSE is smaller for EnSF but it is overall small compared to the scale of the signal. Throughout the test, we make the following observations regarding EnSBF: 
\begin{enumerate}
    \item We introduce a temperature parameter $\beta$ to alleviate likelihood-weight degeneracy, which is substantially more severe under nonlinear observations than in the linear observation setting. Then the weight in the drift term takes the form:
        \begin{align}
        \pi^m(t,x):= \frac{\tilde w^m(t,x)}{\sum^B_{l=1} \tilde w^m(t,x)}, \quad \tilde w^m(t,x) = \exp \Big( \frac{|X^m-a|^2}{2 \beta }-\frac{|X^m-x|^2}{2\beta(1-t)} - \frac{|g(X^m)-y^{\text{obs}}|^2}{2\beta \sigma^2_{\text{obs}}} \Big),
    \end{align}   
    By tempering the likelihood weights, $\beta$ prevents them from becoming overly concentrated and yields a more balanced contribution from the ensemble particles. This stabilization mechanism plays a crucial role in the present example. However, tempering modifies the likelihood and the bridge weights and therefore introduces bias; as a result, the resulting process no longer exactly preserves the original theoretical target distribution. 

    Eventually, we still observe difficulty in extending the EnSBF in its original form with weight tempering to higher dimensions. Some other techniques such as dimension localization could be explored; we leave it as future research. 

    \item Holding the temperature fixed while changing the ensemble size, it is observed that the RMSE decreases steadily as expected, see Figure \ref{fig:comparison_ensemble_temperature}.
\end{enumerate}

\begin{figure}[htbp]
    \centering
    \begin{subfigure}[b]{0.49\textwidth}
        \includegraphics[width=\textwidth]{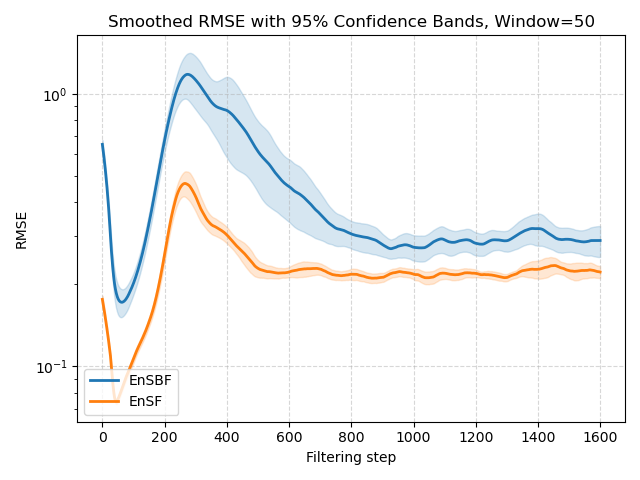}
        \caption{Smoothed RMSE}
        \label{fig:r40sub1}
    \end{subfigure}
    \hspace{-0.6cm}
    \begin{subfigure}[b]{0.52\textwidth}
        \includegraphics[width=\textwidth]{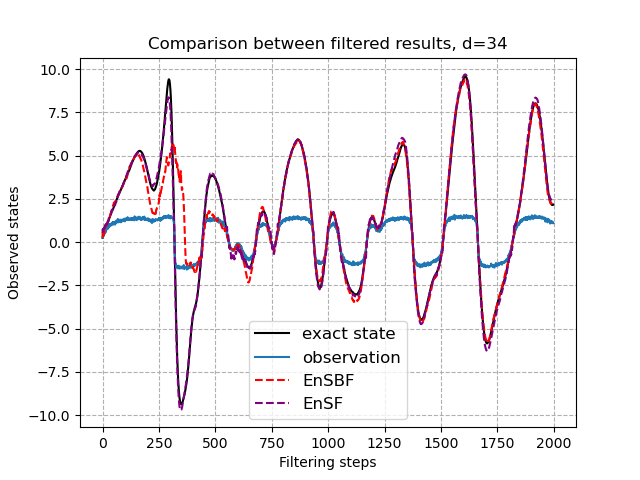}
        \caption{Representative coordinate $i=34$.}
        \label{fig:r40dsub2}
    \end{subfigure}
    \caption{Comparison of state tracking among nonlinear filters, Lorenz-96. Total dimension$=40$, Ensemble size $=128, \beta=40$.}
    \label{fig:tanh-40-d state}
\end{figure}

\begin{figure}[h]
    \centering 
    \includegraphics[width=0.9\textwidth]{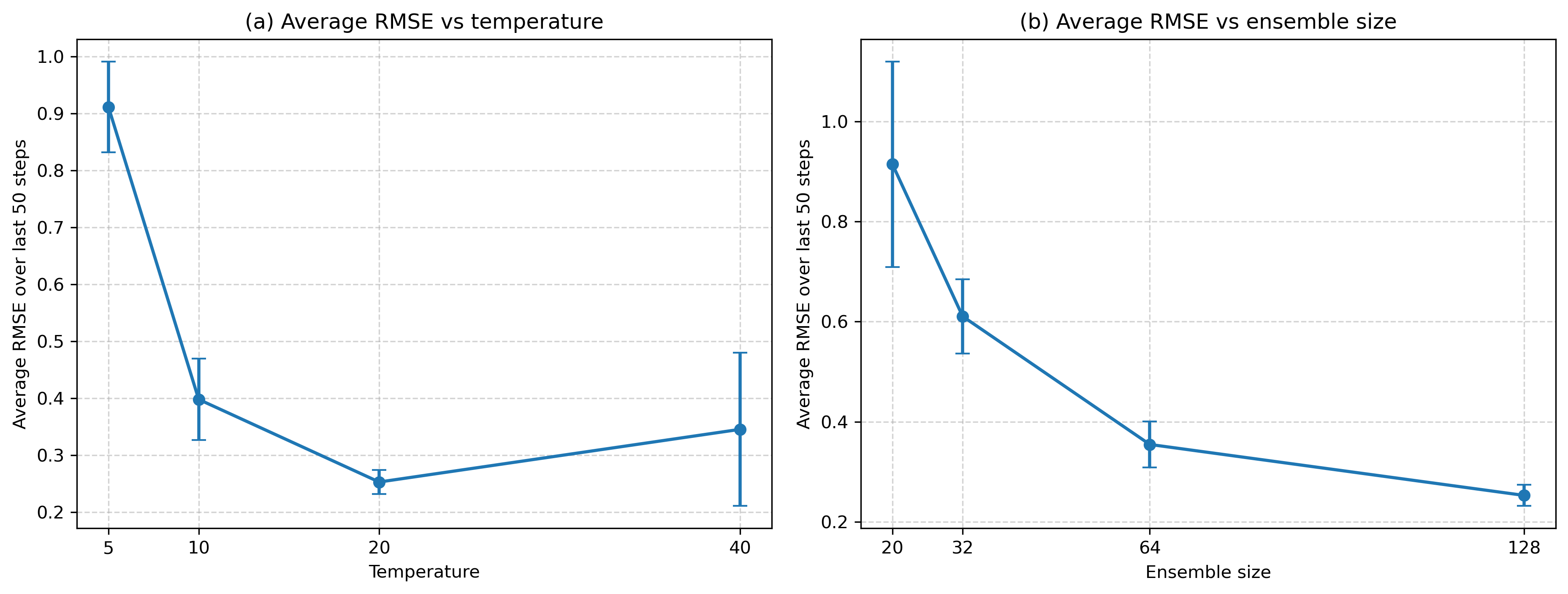} 
    \caption{Comparison of the EnSBF RMSE averaged over the final 50 assimilation steps together with approximate $95\%$ confidence intervals for different temperature and ensemble-size parameters. (a) The ensemble size is fixed at $128$ while the temperature parameter varies. The average RMSE decreases as $\beta$ increases from $5$ to $20$, but rises again at $\beta=40$. (b) The temperature is fixed at $\beta=40$ while the ensemble size varies. The average RMSE decreases monotonically as the ensemble size increases.} 
    \label{fig:comparison_ensemble_temperature} 
\end{figure}

\subsubsection{Example 5, high dimensional example EnSBF with localization.}
In this example, we present empirical analysis where the localization technique is used in EnSBF. We present two test cases, the first one being the $10,000$ dimensional Lorenz-96 model. And the second one is the $1,024$ dimension 1D Kuramoto-Sivashinsky (KS) PDE: 
\begin{align}
    \frac{\partial u}{ \partial t} + \frac{\partial^2 u}{\partial x^2}+ \frac{\partial^4 u}{\partial x^4}+\frac{1}{2} \frac{\partial u^2}{\partial x}=0, \qquad x \in [0,L].
\end{align}
subject to periodic boundary conditions and $L=128 \pi$. The filtering experiment uses DAPPER’s\footnote{\url{https://nansencenter.github.io/DAPPER/}} prescribed initial state followed by a long spin-up before the first assimilation cycle. The dimension $d$ is the number of spatial discretization points. The parameters for the test setup are listed in Table \ref{tab:test-settings}.
For computational efficiency, the approximate score function in the EnSF reverse SDE is evaluated using a minibatch size of one. For both tests, the observation takes the form: 
\begin{align}
     Y_t &= \text{arctan}(X_t) + \gamma_t, \ \gamma_t \sim \mcal{N}(0, \Gamma), \ \ \Gamma = 0.05^2 I. \label{eg5}
\end{align}
An ensemble of size $20$ is used throughout and the likelihood and bridge localization radii are set to $r=0$ and $r'=2$, respectively. Tests were performed for 10 independent trials. All experiments are conducted on a MacBook Air with Apple MPS acceleration, equipped with four performance cores and six efficiency cores. The following observations can be made.
\begin{table}[htbp]
\centering
\caption{Numerical settings for the Lorenz-96 and Kuramoto--Sivashinsky filtering experiments.}
\label{tab:test-settings}
\resizebox{0.7\textwidth}{!}{
\begin{tabular}{lcc}
\toprule
Parameter & Lorenz-96 & Kuramoto--Sivashinsky \\
\midrule
State dimension $d$ & $10,000$ & $1,024$ \\
Forcing/domain parameter & $F=8$ & $\mathrm{DL}\footnote{Dimension less domain-length parameter.}=128$ \\
Signal time step & $\Delta t=0.005$ & $\Delta t_{\mathrm{KS}}=0.25$ \\
Model-noise standard deviation & $0.1$ & $0$ \\
Model steps between observations & $1$ & $4$ \\
Time between observations & $0.005$ & $1.0$ \\
Spin-up length & $1{,}000$ steps & $2{,}150$ ETD--RK4 steps \\
Number of filtering steps & $600$ & $600$ \\
Observation-noise standard deviation & $0.05$ & $0.05$ \\
\midrule
EnSBF artificial-time steps & $25$ & $50$ \\
EnSF reverse-time steps & $100$ & $50$ \\
\bottomrule
\end{tabular}
}
\end{table}

\begin{enumerate}
    \item The localized EnSBF generally requires fewer artificial-time discretization steps than EnSF to achieve comparable or better accuracy. For the $10,000$-dimensional Lorenz-96 model, localized EnSBF attains a lower RMSE using only $25$ artificial-time steps, whereas EnSF uses $100$ reverse-time steps (Figure \ref{fig:10000dsub1}). For the Kuramoto--Sivashinsky equation, both methods use $50$ artificial-time steps; nevertheless, the EnSBF estimate preserves finer spatial structures and exhibits higher resolution, as shown in Figure~\ref{fig:comparison_KS_states}. These results suggest that localization allows EnSBF to exploit the spatial structure of the underlying system more effectively.

    One possible explanation for the difference in the required number of artificial-time steps is that EnSBF initializes the bridge process at the mean of the forecast ensemble, which is already located near the filtering target. In contrast, at each data-assimilation step, EnSF transports a standard Gaussian reference distribution toward the target filtering distribution, and the two distributions may be substantially separated.

    \item Despite its strong state-tracking performance in high dimensions, the computational cost of localized EnSBF remains significant. Its complexity is of order $NB^2d$, where $N$ is the number of artificial-time steps, $B$ is the ensemble size, and $d$ is the state dimension. Figure \ref{fig:time_compare} shows that the run time scales linearly with $d$ while quadratically with $B$. Under the present KS configuration, one data-assimilation cycle requires approximately $0.068$ seconds for EnSBF and $0.033$ seconds for EnSF with minibatch size one. Moreover, numerical experiments indicate that simply doubling the number of reverse-time steps in EnSF does not yield accuracy comparable to that of EnSBF under the same state dimension.

\begin{figure}[h]
    \centering 
    \includegraphics[width=0.75\textwidth]{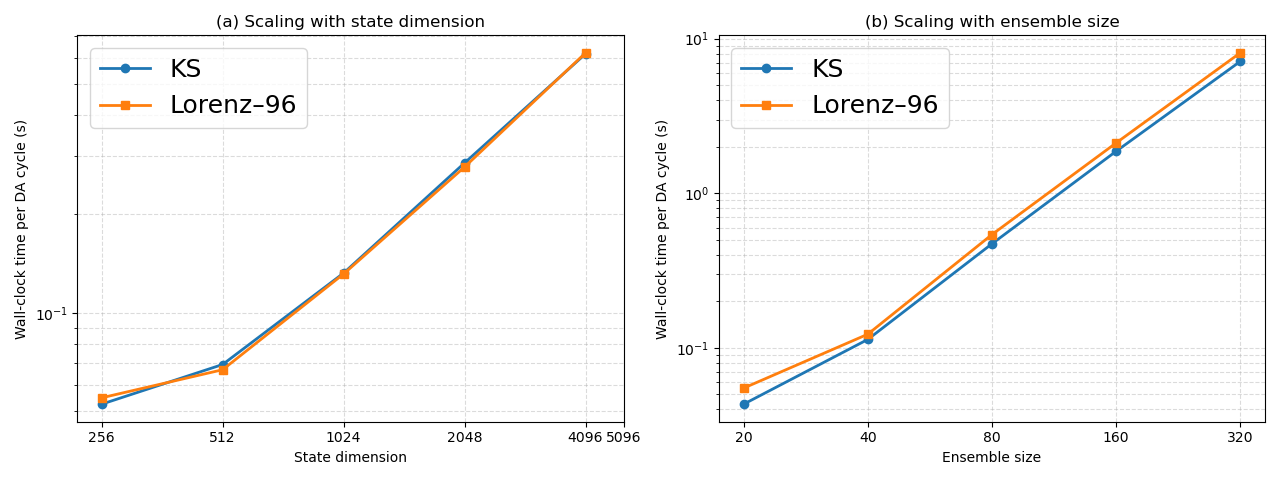} 
    \caption{Wall-clock time study for the KS and Lorenz–96 models, with $N=100$ artificial-time steps used in all experiments. Panel (a) shows the computational cost as the state dimension increases, with the ensemble size fixed at $B=20$. Panel (b) shows the computational cost as the ensemble size increases, with the state dimension fixed at $d=256$.} 
    \label{fig:time_compare} 
\end{figure}

    For the Lorenz--96 model, the two methods achieve broadly comparable accuracy, but the computational difference is more pronounced: EnSBF requires approximately $1.07$ seconds per data-assimilation cycle, whereas EnSF requires only $0.07$ seconds. This increased cost is primarily due to the quadratic dependence of EnSBF on the ensemble size together with the very large state dimension ($d=10,000$).
\end{enumerate}

In summary, although localization introduces approximation bias, the localized EnSBF remains effective for high-dimensional filtering problems and can achieve high accuracy using relatively few artificial-time steps. Nevertheless, its overall computational cost may still exceed that of EnSF, particularly in very high dimensions. The development of more computationally efficient localized EnSBF implementations is therefore left for future work.

\begin{figure}[htbp]
    \centering
    \begin{subfigure}[b]{0.49\textwidth}
        \includegraphics[width=\textwidth]{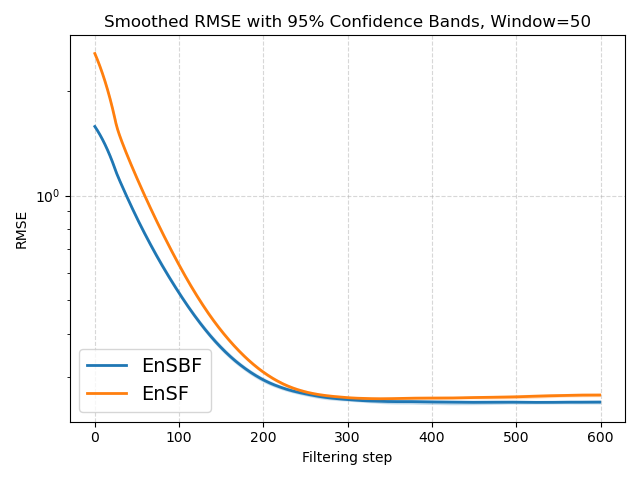}
        \caption{Smoothed RMSE}
        \label{fig:10000dsub1}
    \end{subfigure}
    \hspace{-0.6cm}
    \begin{subfigure}[b]{0.52\textwidth}
        \includegraphics[width=\textwidth]{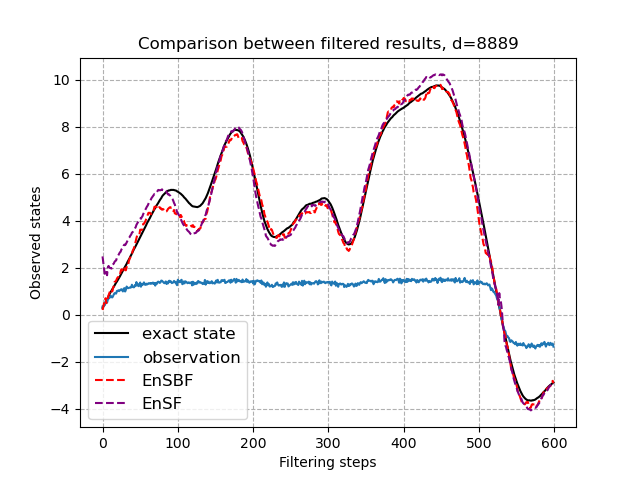}
        \caption{Representative coordinate $i=8889$}
        \label{fig:10000dsub2}
    \end{subfigure}
    \caption{Comparison between localized EnSBF and EnSF, Lorenz-96. Total  $d=10,000$}
    \label{fig:10000-d state}
\end{figure}

\begin{figure}[h]
    \centering 
    \includegraphics[width=0.95\textwidth]{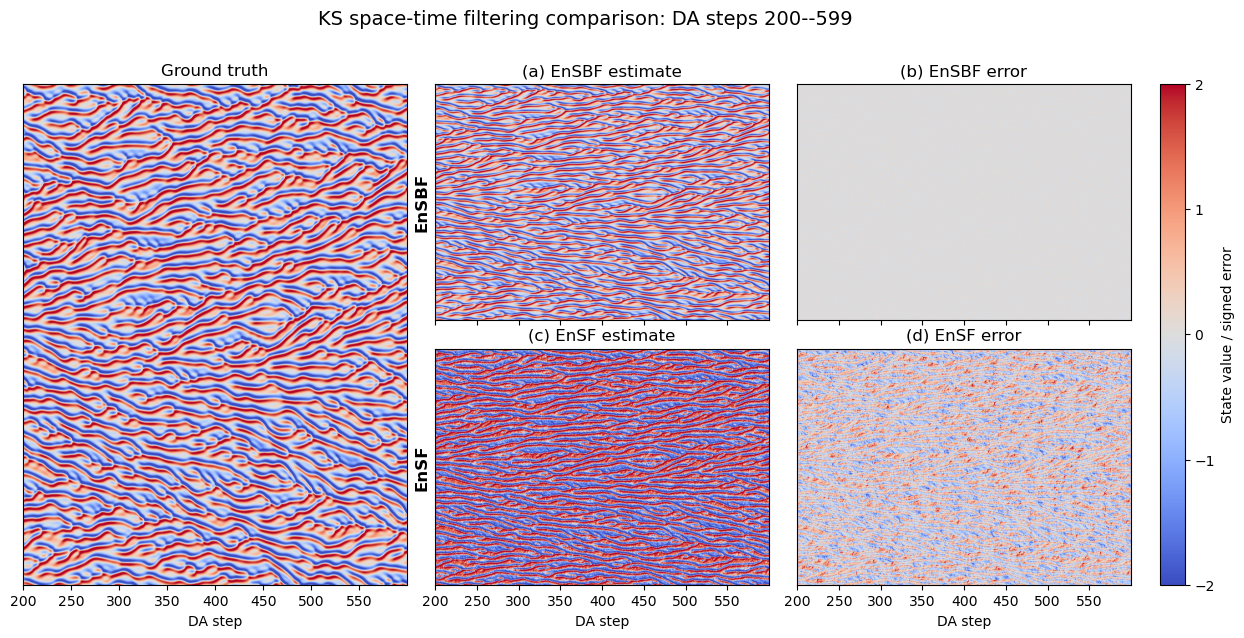} 
    \caption{Comparison between the space-time filtering of KS PDE under the localized EnSBF and the EnSF after the smoothed RMSE stabilizes. Figure on the left shows the ground truth. The top figure in the middle column is the EnSBF estimate and the bottom presents the estimates from EnSF. The top figure in the last column shows the error of EnSBF and bottom figure presents the EnSF error. } 
    \label{fig:comparison_KS_states} 
\end{figure}

\section{Conclusion and future work}
In this work, we developed a novel nonlinear stochastic filtering method, namely the Ensemble Schr{\"o}dinger Bridge Filter. The proposed method demonstrates competitive performance compared with classical benchmark methods, including the Ensemble Kalman Filter and the Particle Filter (PF), in moderately high-dimensional filtering problems with nonlinear and chaotic signal dynamics. The numerical results also suggest that the EnSBF is less sensitive to the ensemble size than the PF. The proposed approach is both training-free and derivative-free. Moreover, in its original form,
EnSBF avoids the posterior score-approximation structural error that appears in the EnSF method \cite{feng1}. In low-dimensional regimes, the EnSBF demonstrates more reliable performance than the EnSF, especially when the observation noise is relatively large.

At present, the main limitation of the original EnSBF is that its performance becomes less stable in very high-dimensional settings so that other techniques such as localization must be introduced. This technique, though effective in state estimation in high dimensions, introduces model structural bias. Future work will therefore focus on refining the analysis procedure so that the method remains effective in very high-dimensional filtering problems while limiting the bias level. We also plan to address the sparse observation limitation and  establish a rigorous convergence theory for EnSBF in its original form.

\newpage
\bibliographystyle{apacite}

\section{Appendix}
\subsection{Discussion on the Schr{\"o}dinger bridge problem and the control formulation.}
The following theorem from \cite{leonard} Theorem 2.8 provides characterizations of the solution to the SB problem. 
\begin{theorem}
Let $\nu, \mu << \mcal{L}_d$, then the SB problem admits a unique solution $\bbP^* = \int f^*(x) g^*(y) d \mathbb{Q}_{\sigma^2}^{xy} dx dy$ where $f^*,g^*$ are $\mcal{L}_d$-measurable  nonnegative functions on $\bR^d$ satisfying the Schr{\"o}dinger system of equations: 
\begin{align}
\begin{cases}
    f^*(x)\E_{\mathbb{Q}_{\sigma^2}}[g^*(X_1)|X_0=x]= \frac{d \nu}{d \mcal{L}_d}(x), &\mcal{L}_d-a.e. \\
    g^*(y)\E_{\mathbb{Q}_{\sigma^2}}[f^*(X_0)|X_1=y]= \frac{d \mu}{d \mcal{L}_d}(y), &\mcal{L}_d-a.e. .
\end{cases}
\end{align}
\end{theorem}

If we denote $f_0(x):=f^*(x)$ and $g_1(y):=g^*(y)$, and that $q(x)=\frac{d \nu}{d \mcal{L}_d}(x)$, $p(x)=\frac{d \mu}{d \mcal{L}_d}(x)$ the density of $\nu$ and $\mu$, the transition density $h_{
\sigma^2
}(s,x,t,y):= (2 \pi \sigma^2 (t-s))^{-\frac{d}{2}} \exp(-\frac{|x-y|^2}{2\sigma^2 (t-s)})$, then we have: 
\begin{align} \label{og_design_seq}
    f_0(x) \int h_{\sigma^2}(0,x,1,y) g_1(y) dy &= q(x) \\
    g_1(y) \int h_{\sigma^2}(0,x,1,y) f_0(x) dx &= p(y). 
\end{align}
Further denoting $\int h_{\sigma^2}(0,x,1,y) g_1(y) dy=g_0(x)$ and $\int h_{\sigma^2}(0,x,1,y) f_0(x) dx=f_1(y)$, we have that the Schr{\"o}dinger system can be characterized by: 
\begin{align}
    q(x)=f_0(x)g_0(x) \nonumber\\
    p(y)=f_1(y)g_1(y)
\end{align}
Defining
\begin{align}
    f_t(x)&:=\int_{\mathbb{R}^d} h_{\sigma^2}(0,z,t,x)f_0(z)\,dz,\\
    g_t(x)&:= \int_{\mathbb{R}^d} h_{\sigma^2}(t,x,1,y)g_1(y)\,dy,
\end{align}
We have the following forward and backward equations on $(0,1) \times \bR^d$ \cite{chen1}:
\begin{align}
\begin{cases}
    \partial_t f_t(x) = \frac{\sigma^2}{2} \Delta f_t(x),  &\nonumber\\
    \partial_t g_t(x) = -\frac{\sigma^2}{2} \Delta g_t(x).
\end{cases}
\end{align}
Following again \cite{chen1}, let $q_t$ denote the density of $\mathbb{P}^*_t$, then this density function is given via: 
\begin{align}
    q_t(x)=f_t(x)g_t(x).
\end{align}
whose solution (optimal control) is given by $ \alpha^*_t(x)=\sigma^2 \nabla_x \log \int h_{\sigma^2}(t,x,1,y) g_1(y) dy$:

Based on the discussion above, we discuss an example with designed initial and target distribution. Denoting $\phi_{\sigma}(x):=\frac{1}{(2 \pi \sigma^2)^{\frac{d}{2}}}\exp(-\frac{|x|^2}{2 \sigma^2})$ the density function of a Gaussian distribution. Assume that $\mu\ll\mathcal{L}_d$, and denote its density by $p$. We take the initial
distribution to be $\nu=\delta_0$. Although this choice is not absolutely
continuous with respect to $\mathcal{L}^d$, it gives a degenerate
Schr{\"o}dinger Bridge formulation that leads to the
Schr{\"o}dinger--F{\"o}llmer sampler, and \eqref{og_design_seq} may be interpreted formally in a measure-valued sense. Taking $f(dx)=\delta_0(x)$, $g_1(y)=\frac{p(y)}{\phi_\sigma(y)}$ we can check that: 
    \begin{align}\label{part1_sys}
        g_1(y) \int h_{\sigma^2}(0,x,1,y) f_0(x) dx &= p(y) \\ 
       f_0(x)\int h_{\sigma^2}(0,x,1,y) g_1(y) dy &= \delta_0(x)
    \end{align}
    since the first equation gives $g_1(y)=\frac{p(y)}{\phi_\sigma(y)}$ and the second equation holds due to the equality $\int h_{\sigma^2}(0,0,1,y) g_1(y) dy=1$. 
Theorem 3.1 and 3.2 in \cite{sb8} shows that the SB problem also has a control formulation which is listed as Problem \ref{problem_control} whose solution (optimal control) is given by $ \alpha^*_t(x)=\sigma^2 \nabla_x \log \int h_{\sigma^2}(t,x,1,y) g_1(y) dy$. 
    
Now set $\sigma=1$.  With a slight abuse of notation, we use the same symbols
$\mu$ and $\mu_W$ to denote both the probability measures and their densities
with respect to $\mathcal{L}_d$. Then by the solution of Problem \ref{problem_control} we have, 
    \begin{align}
        \alpha_t(x)= \nabla_x \log \E_{Z\sim \mu_W}[\frac{\mu}{\mu_W}(x+ \sqrt{T-t} Z)] \label{scheme1_alpha}.
    \end{align}
Therefore, to sample from the target distribution $\mu$, it suffices to start from the initial distribution $\delta_0$ and simulate the controlled dynamics
    \begin{align}\label{scheme_sde}
        dX_t = \alpha_t(X_t)dt + \sigma dW_t , && X_0 \sim \delta_0(x), && T=1. 
    \end{align}

\subsection{Schr{\"o}dinger bridge problem with general reference measure}
For a generalized Schr{\"o}dinger bridge problem, we consider a reference process with time-dependent linear drift and deterministic diffusion coefficient.

\begin{Problem}
(\textbf{Schr{\"o}dinger bridge problem with general coefficients}.)
\label{sb_prob2}
Let $\mathbb Q^x$ denote the law on path space induced by the SDE
\begin{align}
    dX_t = b_t(X_t)\,dt + \sigma_t\,dW_t,
    \qquad X_0=x,
    \label{general_reference_sde}
\end{align}
where $\sigma_t$ is a deterministic, time-dependent function. Given an initial density $q_0$, define the reference path measure $\mathbb Q := \int_{\mathbb R^d} \mathbb Q^x q_0(x)\,dx$. 
The Schr{\"o}dinger bridge problem is to find a probability measure
$\mathbb P^*\in\mathcal P(\Omega)$ such that
\begin{align}
    \mathbb P^*\in \arg\min_{\mathbb P\in\mathcal P(\Omega)} \mathcal H(\mathbb P\mid \mathbb Q),
    \label{general_sb_problem}
\end{align}
subject to the marginal constraints
\[
    \mathbb P_0=\nu,
    \qquad
    \mathbb P_1=\mu.
\]
Here, $\mathbb P_t:=X_t\#\mathbb P$ denotes the time-$t$ marginal of
$\mathbb P$, and the relative entropy is defined by
\begin{align}
\mathcal H(\mathbb P\mid \mathbb Q)
=
\begin{cases}
    \displaystyle
    \int_{\Omega}
    \log\left(\frac{d\mathbb P}{d\mathbb Q}\right)
    d\mathbb P,
    & \text{if } \mathbb P \ll \mathbb Q, \\[1em]
    +\infty,
    & \text{otherwise}.
\end{cases}
\end{align}
\end{Problem}
By the disintegration of measure (\cite{leonard}, A.8) for the relative entropy:
\begin{align}
    \mcal{H}(\mathbb{P}|\mathbb{Q})=\mcal{H}(\mathbb{P}_{01}|\mathbb{Q}_{01})+ \int \mcal{H}(\mathbb{P}^{xy}|\mathbb{Q}^{xy}) \mathbb{P}_{01}(dxdy)
\end{align}
where $\mathbb{P}^{xy}=\mathbb{P}(\cdot | X_0=x, X_1=y)$,$\mathbb{Q}^{xy}=\mathbb{Q}(\cdot | X_0=x, X_1=y)$ are the conditional probabilities, and $\mathbb{P}_{01},\mathbb{Q}_{01}$ are the joint distributions of $(X_0,X_1)$ under the two measures. Since the second term in the entropy disintegration formula is nonnegative, it is minimized by choosing $\mathbb{P}^{xy}=\mathbb{Q}^{xy}$, the original problem \ref{sb_prob2} is reduced to the following static Schr{\"o}dinger Bridge problem for end point coupling. 
\begin{Problem}
(\textbf{Static Schr{\"o}dinger bridge problem}.)
\label{static_sb_prob}
Let $\mathbb Q_{01}\in\mathcal P(\mathbb R^d\times \mathbb R^d)$ denote the
joint distribution of $(X_0,X_1)$ under the reference measure $\mathbb Q$
introduced in Problem \ref{sb_prob2}. Given prescribed marginal distributions
$\nu$ and $\mu$, find
\begin{align}
    \pi^*\in\arg\min_{\pi\in\mathcal P(\mathbb R^d\times\mathbb R^d)}KL(\pi\mid \mathbb Q_{01}), \qquad \pi_0=\nu,\quad \pi_1=\mu,
    \label{static_sb}
\end{align}
where $\pi_0$ and $\pi_1$ denote the first and second marginals of $\pi$,
respectively.
\end{Problem}

The problem is solved by using the Lagrange multipliers, i.e. let: 
\begin{align}
    \mathcal{L}(\pi, \lambda, \gamma):= \int \pi(x,y) \log \frac{\pi(x,y)}{\mathbb{Q}_{01}(x,y)} dx dy + \int \lambda(x) \big(\int \pi(x,y)dy - \rho_{\nu}(x) \big) dx + \int \gamma(y) \big(\int \pi(x,y)dx - \rho_{\mu}(y) \big)dy
\end{align}
Setting the first variation to 0 gives: 
\begin{align}
    1+\log \pi(x,y) -\log \mathbb{Q}_{01}(x,y) + \lambda(x)+\gamma(y)=0.
\end{align}
Noting that $\mathbb{Q}_{01}(x,y)=q(0,x,1,y)q_0(x)$ where $q(0,x,1,y)$ is the transition density. After rearranging, we obtain 
\begin{align}
    \pi(x,y) = f_0(x)q(0,x,1,y)g_1(y) \label{eq_pi}
\end{align}
for some functions $f_0(x)$ and $g_1(y)$ which satisfy the following integral constraint:
\begin{align}
    f_0(x) \int q(0,x,1,y) g_1(y) dy &= \rho_{\nu}(x) \label{system_21} \\ 
    g_1(y)  \int q(0,x,1,y) f_0(x) dx &= \rho_{\mu}(y). \label{system_22}
\end{align}
Defining 
\begin{align}
    g_0(x)&:=\int q(0,x,1,y) g_1(y) dy \\ 
    f_1(y)&:=\int q(0,x,1,y) f_0(x) dx, 
\end{align}
we then obtain the following system of equations: 
\begin{align}
    f_0(x)g_0(x)&=\rho_{\nu}(x) \\ 
    g_1(y)f_1(y)&=\rho_{\mu}(y). 
\end{align}
Thus, using \eqref{eq_pi} the measure defined by 
\begin{align}
    \mathbb{P}^*(\cdot) = \int \mathbb{Q}^{xy}(\cdot) \pi(x,y)dxdy
\end{align}
solves the Schr{\"o}dinger bridge problem \ref{sb_prob2}.

Finally, similar to the earlier construction, we want to correct the drift in the reference SDE so that the two marginal distributions are interpolated via the corrected process. Define
\begin{align}
h(t,x)&:=\int_{\mathbb R^d} q(t,x,1,y)g_1(y)\,dy.
\end{align}
According to Theorem 3.2 in \cite{sb8}, in the Brownian reference case with $q$ being the Gaussian transition kernel, $\sigma=I_d$, the optimal drift correction is given by 
\begin{align}
    \alpha(t,x) =\nabla_x\log h(t,x) = \nabla_x
    \log \int_{\mathbb R^d} q(t,x,1,y)g_1(y)\,dy .
    \label{control2}
\end{align}
More generally, when the reference process has diffusion coefficient $\sigma_t$, the controlled drift correction is given by
\begin{align}
     \alpha(t,x) = \sigma_t\sigma_t^\top \nabla_x \log h(t,x).
\end{align}
And this drift correction solves the corresponding stochastic optimal control formulation of the Schr{\"o}dinger bridge problem.

\newpage

\end{document}